%% file: main.tex
\documentclass[%
  onecolumn 
   , colorlinks 
]{mpi2015-cscpreprint}

\usepackage[american]{babel}

\usepackage{graphicx}

\usepackage{amssymb}
\usepackage{amsthm}

\usepackage[vlined,longend,linesnumbered,ruled]{algorithm2e}
\usepackage{mathabx}

\usepackage{subcaption}

\usepackage{multirow}

\usepackage[colorinlistoftodos,prependcaption]{todonotes}

\usepackage{comment}

\usepackage{tikz}
\usepackage{standalone}
\usepackage{import}

\newcommand{\R}{\mathbb{R}}
\newcommand{\N}{\mathbb{N}}
\newcommand{\norm}[1]{\left\lVert#1\right\rVert}
\newcommand{\normC}[1]{\left\lVert#1\right\rVert_{C(W;\mathbb{R}^n)}}
\newcommand{\normCr}[1]{\left\lVert#1\right\rVert_{C^r(W;\mathbb{R}^n)}}
\newcommand{\normt}[1]{\left\lVert#1\right\rVert_{C^r([0,T]\times W;\mathbb{R}^n)}}

\newcommand{\normtd}[1]{\left\lVert#1\right\rVert_{C^r([0,T]\times W;\mathbb{R}^{2d})}}
\newcommand{\normCrmax}[1]{\left\lVert#1\right\rVert_{C_{\max}^r(W;\mathbb{R}^n)}}
\newcommand{\normtmax}[1]{\left\lVert#1\right\rVert_{C_{\max}^r([0,T]\times W;\mathbb{R}^n)}}
\newcommand{\strich}{\,\middle\vert\,}
\newcommand{\abs}[1]{\left\lvert#1\right\rvert}
\newcommand{\pq}{\begin{pmatrix}p \\ q \end{pmatrix}}

\DeclareMathOperator*{\diag}{diag}
\DeclareMathOperator*{\dist}{dist}

\usepackage{cleveref}

\newtheorem{lem}{Lemma}
\newtheorem{theorem}{Theorem}
\newtheorem{cor}{Corollary}
\theoremstyle{definition}
\newtheorem{defi}{Definition}
\crefname{defi}{Definition}{Definitions}

\begin{document}
  

\title{Time-adaptive SympNets for separable Hamiltonian systems}
  
\author[$\ast, \dagger$]{Konrad Janik}
\author[$\ast, \ast \ast, \ddagger$]{Peter Benner}
\affil[$\ast$]{Max Planck Institute for Dynamics of Complex Technical Systems, Sandtorstraße 1, 39106 Magdeburg, Germany.%
}%
\affil[$\ast\ast$]{
Otto von Guericke University Magdeburg, Fakultät für Mathematik, Universitätsplatz 2, 39106 Magdeburg, Germany.%
}
 
\affil[$\dagger$]{%
  \email{janik@mpi-magdeburg.mpg.de}, \orcid{0009-0004-9030-0708}%
  }

\affil[$\ddagger$]{%
    \email{benner@mpi-magdeburg.mpg.de}, \orcid{0000-0003-3362-4103}%
    }

\shorttitle{TSympNets}
\shortauthor{K. Janik, P. Benner}
\shortdate{}
  
\keywords{Hamiltonian systems, symplectic integrators, neural networks, time stepping}

\abstract{%
    Measurement data is often sampled irregularly i.e. not on equidistant time grids. This is also true for Hamiltonian systems. However, existing machine learning methods, which learn symplectic integrators, such as SympNets \cite{JinZZetal20} and HénonNets \cite{BurTM21} still require training data generated by fixed step sizes. To learn time-adaptive symplectic integrators, an extension to SympNets, which we call TSympNets, was introduced in \cite{JinZZetal20}. We adapt the architecture of TSympNets and extend them to non-autonomous Hamiltonian systems. So far the approximation qualities of TSympNets were unknown. We close this gap by providing a universal approximation theorem for separable Hamiltonian systems and show that it is not possible to extend it to non-separable Hamiltonian systems. 
    To investigate these theoretical approximation capabilities, we perform different numerical experiments.
    Furthermore we fix a mistake in a proof of a substantial theorem \cite[Theorem 2]{Tur03} for the approximation of symplectic maps in general, but specifically for symplectic machine learning methods.
}


\novelty{
\begin{itemize}
    \item We correct the proof of the symplectic polynomial approximation theorem \cite[Theorem 2]{Tur03}, which shows that every symplectic map can be approximated by Hénon-like maps. This theorem is the basis for symplectic machine learning methods like SympNets \cite{JinZZetal20}, HénonNets \cite{BurTM21} and generalized Hamiltonian neural networks \cite{HorSKetal25}.
    \item We propose new SympNet architectures that are suitable for adaptive time steps and non-autonomous Hamiltonian systems.
    \item We prove universal approximation theorems for those SympNets on separable Hamiltonian systems.
    \item We also discuss the limitations of the old and our new SympNets.
\end{itemize}
}

\maketitle
%

\input{introduction}

\input{preliminaries}

\input{correction}

\input{theory}

\input{experiments}

\input{conclusion}


\addcontentsline{toc}{section}{References}
\bibliographystyle{plainurl}
\bibliography{exampleref}
  
\end{document}

%% file: introduction.tex
\section{Introduction}
The numerical integration of Hamiltonian systems is a central topic in computational physics and theoretical chemistry. Concrete applications are found in celestial mechanics, molecular design, and many more \cite{LeiR05,HaiLW06,Arn13}. Though the topic has been well researched, in particular since the 1990ies, where many structured integration schemes were developed, novel interest arose with the advent of machine learning and artificial intelligence for scientific computing. This allows in particular to derive compact Hamiltonian models from data --- be it experimental or computed via simulation using a legacy code --- via learning the Hamiltonian \cite{GreDY19,BerDMetal19} or the Lagrangian in the position-velocity space \cite{CranGHetal20}. There are many extensions to these Hamiltonian neural networks (HNNs) \cite{CheZAetal20, DavM23, HorSKetal25, XioTHetal21, YilGBetal23, HanGH21}. While the original approach still required derivative data, the neural ODE method \cite{CheRBetal18} was quickly implemented with symplectic integrators to preserve the structure of the Hamiltonian system \cite{DavM23, CheZAetal20}. Because of the implicit nature of symplectic integration schemes \cite{HaiLW06} the fast evaluation of the learned models was limited to separable Hamiltonian systems $H(p,q)=K(p)+V(q)$. To tackle this problem, phase space extensions were introduced in \cite{Tao16, XioTHetal21}. Additionally, an approach on how to handle parametric Hamiltonian systems was given in \cite{HanGH21}. Another way of learning Hamiltonian systems is the identification of proper generating functions with neural networks \cite{CheT21}.
Yet another approach consists of learning a suitable symplectic integrator for a Hamiltonian systems using deep learning, namely SympNets \cite{JinZZetal20}, HénonNets \cite{BurTM21} and generalized Hamiltonian neural networks \cite{HorSKetal25}. SympNets are designed as a composition of different shea-maps see e.g. \Cref{subsec:td_SympNet} or \cite{JinZZetal20}), while HénonNets are compositions of Hénon-like maps. Both of these architectures fall into the framework of generalized HNNs as compositions of symplectic integrators evaluated for different separable Hamiltonians in each layer. All of these methods are universal approximators of symplectic maps and use the symplectic polynomial approximation theorem \cite[Theorem 2]{Tur03} to prove that. Unfortunately, there is a mistake in the proof of \cite[Theorem 2]{Tur03}, which we will explain and fix in \Cref{sec:corr}. So the universal approximation capabilities of the above methods still hold true. One of the issues with SympNets and HénonNets is that they require fixed step sizes. This strongly restricts the training data thatFor SympNets a time-adaptive extension was proposed in \cite{JinZZetal20} but not analized extensively. The aim of this paper is to close the gap between explicit symplectic integrators based on arbitrary time grids and SympNets with fixed time grid demand, as simulation data could come from such time-adaptive symplectic integration schemes, but could not used in that form for training SympNets. Thus, we develop time-adaptive SympNets and discuss their theoretical properties as universal approximators. Additionally, we extend these time-adaptive to non-autonomous Hamiltonian systems, where the Hamiltonian $H(p,q,t)$ explicitely depends on time. Since SympNets are not only used for learning Hamiltonian systems, but also for volume preserving dynamics \cite{Baj23} and Poisson systems \cite{JinZKetal22}, our results may contribute in those applications as well.\\
The rest of the paper is structured as follows. We will make some necessary theoretical preparations in \Cref{sec:prep}. In particular we will introduce some notation in \Cref{subsec:notation} and recall a generalized version of the Faá-di-Bruno Formula in  \Cref{subsec:faa}. In \Cref{sec:corr} we state the symplectic polynomial approximation theorem \cite[Theorem 2]{Tur03} and correct a mistake in the proof, by showing \Cref{the:comp}, which is a nice result about estimating the difference between compositions of functions in the $C^r$-norm, in general. In \Cref{sec:app} we will quickly recap SympNets and proof our main result \Cref{the:g-sympnet}, a universal approximation theorem for time-adaptive SympNets.  In \Cref{subsec:lim,subsec:lim_og_LA} we discuss the limitation of the current version of time-adaptive SympNets and show why they can only be applied to separable Hamiltonian systems. In \Cref{sec:exp} we show that our theoretical findings also have practical implications by testing the previously discussed methods on three different examples, i.e., the mathematical pendulum in \Cref{subsec:pend}, a linear but non-separabel example in \Cref{subsec:lin} and the harmonic oscillator with an external force as an example of a non-autonomous Hamiltonian system in \Cref{subsec:ho}. Lastly, we sum up our results in \Cref{sec:conc} and discuss future research possibilities.

%% file: preliminaries.tex
\section{Preliminaries}
In this section we will establish some results, which will be useful to prove our theoretical results. We begin with some notation, we will use through out this work.
\label{sec:prep}
\subsection{Notation}
\label{subsec:notation}
Let $d \in \N$. We start by introducing some notation for multi-indices $\alpha \in \N_0^d$. We define the absolute value and the factorial of a multi-index $\alpha$ as follows
\begin{align*}
	\abs{\alpha}&:= \sum_{i=1}^{d} \alpha_i, \qquad
	\alpha ! := \prod_{i=1}^{d} (\alpha_i !).
\end{align*}
Furthermore, a vector $v \in \R^d$ to the power of $\alpha$ is given by
\begin{align*}
	v^\alpha:=\prod_{i=1}^{d}v_i^{\alpha_i}
\end{align*} 
Additionally, we define the order $\prec$ on $\N_0^d$. Let $\alpha, \beta \in \N_0^d$, we write $\alpha \prec \beta$ if and only if one of the following conditions holds
\begin{enumerate}
	\item $\abs{\alpha} < \abs{\beta}$,
	\item $\abs{\alpha}=\abs{\beta} \text{ and } \alpha_1 <\beta_1$,
	\item $\abs{\alpha}=\abs{\beta}, \alpha_1=\beta_1,...,\alpha_k=\beta_k \text{ and }\alpha_{k+1}<\beta_{k+1} \text{ for some }1\leq k <d$
\end{enumerate}
Let $n \in \N, r \in \N_0, W \subseteq \R^d$ be at compact set. We define the norm on $C^r(W;\R^n)$ by
\begin{align*}
	\normCr{f}:= \sum_{\abs{\alpha}\leq r} \max_{1 \leq i \leq n} \sup_{x \in W}\abs{D^\alpha f^{(i)}(x)},
\end{align*}
with
\begin{align*}
	D^\alpha f := \frac{\partial^{\abs{\alpha}}f}{\partial x_1^{\alpha_1} \dotsm \partial x_d^{\alpha_d}}
\end{align*}
and $f=(f^{(1)},...,f^{(n)})^T$. As it is usually done, we just write $C^r(W)$ for $n=1$. For technical reasons we will also need a slightly different norm on $C^r(W;\R^n)$ given by
\begin{align*}
	\normCrmax{f}:=\max_{\abs{\alpha}\leq r}\max_{1 \leq i \leq n} \sup_{x \in W}\abs{D^\alpha f^{(i)}(x)}.
\end{align*}
When we are looking at time dependent functions $f:[0,T]\times W \to \R^n$ with $T>0$, we will be using the standard $C^r$ norm given by
\begin{align*}
	\normt{f}:=\sum_{\abs{\alpha}\leq r} \max_{1\leq i \leq n} \sup_{(t,x) \in [0,T]\times W} \abs{D^\alpha f^{(i)}(t,x)}.
\end{align*}
Note that in this case $D^\alpha f^{(i)}$ also includes time-derivatives.\\
Let $\mathcal{I}, \mathcal{J}$ be two finite index sets and $M_{i,j} \subseteq M$ for $i \in \mathcal{I}, j \in \mathcal{J}$, where $M$ is an arbitrary set. It will be useful to define the combination of two Cartesian products as follows
\begin{align*}
	\left[\bigtimes_{i \in \mathcal{I}} \bigtimes_{j \in \mathcal{J}}\right]M_{i,j}:=M_{i_1,j_1}\times... \times M_{i_1,j_{\abs{\mathcal{J}}}}\times...\times M_{i_{\abs{\mathcal{I}}},j_1}\times ... \times M_{i_{\abs{\mathcal{I}}},j_{\abs{\mathcal{J}}}}\subseteq M^{\abs{\mathcal{I}}\cdot\abs{\mathcal{J}}}
\end{align*}
Furthermore we denote the Elements of $\left[\bigtimes_{i \in \mathcal{I}} \bigtimes_{j \in \mathcal{J}}\right]M_{i,j}$ as $[a_{i,j}]_{i\in \mathcal{I},j\in \mathcal{J}}$, where $a_{i,j} \in M_{i,j}$ for $i \in \mathcal{I}, j \in \mathcal{J}$. So if we consider a function $f:\R^{\abs{\mathcal{I}}\cdot\abs{\mathcal{J}}} \to \R$ and a point $x=(x_{i_1,j_1},...,x_{i_{\abs{\mathcal{I}}},j_{\abs{\mathcal{J}}}})^T \in \R^{\abs{\mathcal{I}}\cdot\abs{\mathcal{J}}}$, we write
\begin{align*}
	f\left(\left[x_{i,j}\right]_{i \in \mathcal{I},j \in \mathcal{J}}\right):=f(x)
\end{align*}
for the evaluation of $f$ at $x$. When it is obvious from the context what $\mathcal{I}$ and $\mathcal{J}$ are, we will just write $[x_{i,j}]_{i,j}$ to make the notation more readable.\\
Furthermore, we want to denote closed balls around a set $M\subseteq \R^n$ with radius $R$ by
\begin{align*}
	\overline{B_R(M)}:=\left\{x \in \R^n \vert \exists y \in M: \norm{x-y}_\infty \leq R\right\}.
\end{align*}
The next notation we want to introduce is that of matrix-like functions, which we will use to define SympNets.
As in \cite{JinZZetal20}, we will denote them with
\begin{align*}
	\begin{bmatrix}
		f_1 & f_2 \\
		f_3 & f_4
	\end{bmatrix}: \R^{2d} \to \R^{2d}, \qquad
	\begin{bmatrix}
		f_1 & f_2 \\
		f_3 & f_4
	\end{bmatrix}
	\begin{pmatrix}
		p\\q
	\end{pmatrix}:=
	\begin{pmatrix}
		f_1(p)+f_2(q)\\
		f_3(p)+f_4(q)
	\end{pmatrix},
\end{align*}
where $f_i:\R^{d} \to \R^d,~i=1,...,4$.\\
Lastly, we write the set of neural networks with one hidden layer and activation function $\sigma$ as
\begin{align*}
	\Sigma_d(\sigma)=\left\{f: \R^n \mapsto \R \strich f(x)=a^T \sigma(K x +b), a,b \in \R^m, K \in \R^{m\times n}, m \in \N\right\}
\end{align*}
like in \cite{HorSW90}.
\subsection{Generalized Faá-di-Bruno Formula}
We will need the generalized Faá-di-Bruno formula \cite{ConS96} a lot in \Cref{sec:corr}. Therefore, we state it here and derive some consequences. 
\label{subsec:faa}
\begin{theorem}[Faá-di-Bruno Formula \cite{ConS96}]
	\label{the:faa}
	Let $r, n_1, n_2 \in \N$,~$U \subseteq \R^{n_1}$ open, $g\in C^r(U;\R^{n_2})$ and $f \in C^r(g(U))$. For all $\alpha \in \N_0^{n_1}, 1\leq\abs{\alpha}\leq n$ and $x \in U$ we have
	\begin{align*}
		D^\alpha(f \circ g)(x)=\sum_{\abs{\beta}=1}^{\abs{\alpha}}
		\left(D^\beta f\right)(g(x))
		\sum_{s=1}^{\abs{\alpha}} \sum_{p_s(\alpha,\beta)} (\alpha !)
		\prod_{j=1}^{s} \frac{[D^{l_j}g(x)]^{k_j}}{(k_j!)(l_j!)^{\abs{k_j}}}
	\end{align*}
	with
	\begin{align*}
		p_s(\alpha, \beta)&=\left.\Bigg\{(k_1,...,k_s;l_1,...,l_s)  \strich k_i \in \N_0^{n_2}, \abs{k_i}>0, l_i \in \N_0^{n_1},\abs{l_i }>0\right.\\
		&\left. \qquad \qquad \qquad\qquad\qquad, l_1 \prec ... \prec l_s, \sum_{i=1}^{s}k_i=\beta, \sum_{i=1}^{s}\abs{k_i}l_i=\alpha\right\}.
	\end{align*}
\end{theorem}
A less explicit (but easier to write and for most of our purposes sufficient) way of stating \Cref{the:faa} is the following corollary. 
\begin{cor}
	\label{the:faacor}
	Let $r, n \in \N$,~$U \subseteq \R^{n}$ open, $g\in C^r(U;\R^{n})$ and $f \in C^r(g(U))$. For all $\alpha \in \N_0^{n}, \abs{\alpha}\leq r$ there are polynomials $P_{\alpha, \beta, n}$ only depending on $\alpha, \beta$ and $n$ such that
	\begin{align}
		\label{eq:cor}
		D^\alpha(f \circ g)(x)= \sum_{\abs{\beta}=0}^{\abs{\alpha}} \left(D^\beta f\right)(g(x)) P_{\alpha, \beta, n}\left(\left[D^\gamma g^{(i)}(x)\right]_{\gamma \in \N_0^r, i\in \{1,...,n\}}\right)
	\end{align}
	for all $x \in U$.
\end{cor}
\begin{proof}
	For $\abs{\alpha}>0$ we choose $P_{\alpha,0,n}\equiv 0$ and \cref{eq:cor} follows immediately from \Cref{the:faa}. In the trivial case $\abs{\alpha}=0$ we choose $P_{0,0,n}\equiv 1$.  
\end{proof}
For easier notation we will just write $\gamma, i$ instead of $\gamma \in \N_0^r, i\in \{1,...,n\}$.
Multiindices with absolute value $1$, i.e., $\abs{\beta}=1$ will play a special role in \Cref{sec:corr}. Therefore, we take a closer look at $P_{\alpha,\beta,n}$ for $\abs{\beta}=1$ in the following lemma.
\begin{lem}
	\label{lem:beta=1}
	Let $r, n \in \N,~\alpha, \beta \in \N_0^r,~\abs{\alpha}>0,~\abs{\beta}=1$ and $W \subseteq \R^{n}$ be compact. Then for $P_{\alpha,\beta,n}$ from \Cref{the:faacor} we have
	\begin{align*}
		\abs{P_{\alpha,\beta,n}\left(\left[D^\gamma g_1^{(i)}(x)\right]_{\gamma,i}\right)-
			P_{\alpha,\beta,n}\left(\left[D^\gamma g_2^{(i)}(x)\right]_{\gamma,i}\right)
		}\leq \normCrmax{g_1-g_2}
	\end{align*}
	for all $i=1,...,n,~x \in W$ and $g_1,g_2 \in C^r(W;\R^n)$.
\end{lem}
\begin{proof}
	Since $\abs{\beta}=1$, there is $j \in \{1,...,n\}$ such that $\beta=e_j$, with $e_j$ being the $j$-th unit vector. Looking at \Cref{the:faa}, we observe $p_s(\alpha, \beta)=\emptyset$ for $s\geq 2$ and $\abs{\beta}=1$, which yields
	\begin{align*}
		P_{\alpha,\beta,n}\left(\left[D^\gamma g^{(i)}(x)\right]_{\gamma,i}\right)= (\alpha!) \frac{\left[D^\alpha g(x)\right]^\beta}{(\beta !) (\alpha !)^{\abs{\beta}}}=D^\alpha g^{(j)}(x)
	\end{align*}
	for any $g \in C^r(W;\R^n)$. Therefore we have
	\begin{align*}
		\abs{P_{\alpha,\beta,n}\left(\left[D^\gamma g_1^{(i)}(x)\right]_{\gamma,i}\right)
			P_{\alpha,\beta,n}\left(\left[D^\gamma g_2^{(i)}(x)\right]_{\gamma,i}\right)}=\abs{D^\alpha g_1^{(j)}- D^\alpha g_2^{(j)} } \leq \normCrmax{g_1 -g_2}.
	\end{align*}
\end{proof}
With these preparations, we are ready to correct the proof of the symplectic polynomial approximation theorem.

%% file: correction.tex
\section{Correction of the symplectic polynomial approximation theorem}
\label{sec:corr}
In this section we correct a mistake in the proof of \cite[Theorem 2]{Tur03}. First, we state the theorem.
\begin{theorem}[Symplectic polynomial approximation {\cite[Theorem 2]{Tur03}}]
	\label{the:turaev2}
	Let $r, d \in \N,~ U \in \R^{2d}$ an open Ball and $F \in C^r(U;\R^{2d})$ be symplectic. For all $\epsilon > 0$ and compact $W\subseteq U$ there are $m \in \N$ and polynomial Hénon-like maps $f_1,...f_{4m}$ such that
	\begin{align*}
		\normCr{F-f_{4m}\circ ... \circ f_1}\leq \epsilon.
	\end{align*}
\end{theorem}
As explained below, the proof \Cref{the:turaev2} contains a mistake, i.e., \cref{eqn:turaev_false}. We fix this mistake by replacing \cref{eqn:turaev_false} with \Cref{the:compcor} and observing that all conditions of \Cref{the:compcor} are fulfilled. Other than that one simply follows the arguments in \cite{Tur03}.\\
Note that the norm used in \cite{Tur03} 
\begin{align*}
	\norm{f}_{C^r(W;\R^{2d}),\text{tur}}=\max_{x \in W} \sum_{\abs{\alpha}\leq r} \norm{D^\alpha f(x)}
\end{align*}
is a little bit different from the one we use. But we will show that \Cref{the:turaev2} still holds in $\normCr{\cdot}$ and since $\norm{f}_{C^r(W;\R^{2d}),\text{tur}}\leq \norm{f}_{C^r(W;\R^{2d})}$ for all $f \in C^r(W;\R^{2d})$ it also holds for the norm used in \cite{Tur03}.\\
The problem with the proof in \cite{Tur03} is that it uses the following inequality
\begin{align}\label{eqn:turaev_false}
	\norm{F_{0,1}-G_{m-1}\circ ... \circ G_0}\leq \sum_{l=1}^{m-1}\norm{F'_{l/m,1-l/m}}\norm{F_{(l-1)/m,1/m}-G_{l-1}}
\end{align}
with $F_{0,1}=F_{1-1/m,1/m}\circ ... \circ F_{0,1/m}$. This inequality is missing $\norm{F_{m-1/m,1/m}-G_{m-1}}$ on the right hand side for $\norm{\cdot}=\norm{\cdot}_{C^0(W;\R^n)}$ and follows from the mean value theorem, if one adjusts the domains used in the norms accordingly. But even with this simple correction, \cref{eqn:turaev_false} is not true for $\norm{\cdot}_{C^r(W;\R^{2d})},~r \in \N$. An easy counterexample to \cref{eqn:turaev_false} is choosing $r=1,~m=2,~W=[0,1]$ and $F_{0,\frac{1}{2}}(x):=x^2-1,~G_0(x)=G_1(x)=F_{\frac{1}{2},\frac{1}{2}}(x)=x^2$. This example leads to
\begin{align*}
	\norm{F_{\frac{1}{2},\frac{1}{2}} \circ F_{0,\frac{1}{2}} - G_1 \circ G_0}_{C^1([0,1])}&=5 > 4\\ &= \norm{F_{\frac{1}{2},\frac{1}{2}}'}_{C^1([-1,1])} \norm{F_{0,\frac{1}{2}} -G_0}_{C^1([0,1])} + \norm{F_{\frac{1}{2},\frac{1}{2}}-G_1}_{C^1([-1,1])}
\end{align*}
with the same results if we use $\norm{\cdot}_{C^1([0,1];\R),\text{tur}}$. The rest of this section is dedicated to fixing this problem, which we do in a more generalized way by showing \Cref{the:comp}. To prepare \Cref{the:comp}, we need the following lemmata. 

\begin{lem}
	\label{lem:Wtilde}
	Let $n, m\in \N,~ U \subseteq \R^n$ open, $f \in C^1(U ;\R^n)$ and $\Phi:\R^+_0 \times U \to \R^n$ be the phase flow of the ODE
	\begin{align*}
		\dot{x}(t)=f(x), \qquad x(0)=x_0,
	\end{align*}
	i.e., $\Phi(t,x_0)=x(t)$. Let $T>0$ such that $\Phi(t,x)$ exists for all $t \in [0,T],~x \in U$ and $\Phi([0,T]\times U) \subseteq U$. We define $\Phi_{i,m}(h,x):=\Phi(ih/m,x)$. Furthermore, let $g_i \in C([0,T]\times\R^n;\R^n)$, $g_{i,m}(h,x):=g_j(h/m,x)$ for and $G_{i,m}(h,x):=g_{i,m}(h,g_{i-1,m}(h,...g_{1,m}(h,x)...))$ for $i=1,..,m$. For a compact set $W \subseteq U$ we define $W_\Phi:=\Phi([0,T],W)$ and $\tilde{W}:=\overline{B_R(W_\Phi)}$. with $R < \dist(W_\Phi,\partial U)$. Then there is a constant $C_\Phi$ such that for all $h \in [0,T],~x,y \in \tilde{W}_\Phi$
	\begin{align}
		\norm{\Phi(h/m,x)-\Phi(h/m,y)}_\infty \leq \left(1 + \frac{C_\Phi}{m}\right) \norm{x-y}_\infty. \label{eqn:LPhi}
	\end{align}
	Assume that there is a constant $C < \infty$ independent of $m$ such that for all $j=1,...,m$
	\begin{align*}
		\norm{g_{j,m}-\Phi_{1,m}}_{C([0,T]\times\tilde{W};\R^n)}&\leq \frac{C}{m^2}.
	\end{align*}
	Let $m_0:=\lceil C \exp(C_\Phi)/R \rceil$. then we have
	\begin{align}
		\label{eqn:Wtilde}
		G_{i,m}([0,T],W) \subseteq \tilde{W}
	\end{align}
	for all $m \geq m_0$ and $i\leq m$.
\end{lem}

\begin{proof}
	First, we show \cref{eqn:LPhi}. Since $f \in C^1$ this also holds for $\Phi$ with respect to $x$. see, for example, \cite[Appendix B, Lemma B.4]{DuiK12}. And because of
	\begin{align*}
		\partial_t \Phi(t,x_0)=\dot{x}(t)=f(x(t)),
	\end{align*}
	$\partial_t \Phi$ is continuously differentiable as well. Therefore, and because $\tilde{W}_\Phi\subseteq U$ is compact, $\partial_t \Phi$ is Lipschitz continuous wth respect to $x$ on $\tilde{W}$, with Lipischitz constant $L_\Phi$, i.e., 
	\begin{align*}
        \norm{\partial_t \Phi(h, x)-\partial_t \Phi(h,y)}_\infty\leq L_\Phi \norm{x-y}_\infty,
	\end{align*}
	for $x,y\in \tilde{W},~h \in [0,T]$. By choosing $C_\Phi:=TL_\Phi$, we obtain exists and for $h \in [0,T], x,y \in \tilde{W}_\Phi$ we have
	\begin{align*}
		\norm{\Phi(h/m,x)-\Phi(h/m,y)}_\infty \leq& \norm{x + \frac{h}{m}\partial_t \Phi(\theta,x)-y - \frac{h}{m}\partial_t \Phi(\theta,y)}_\infty\\
		\leq&\left(1+ \frac{C_\Phi}{m}\right)\norm{x-y}_\infty,
	\end{align*}
	for $x,y\in \tilde{W},~h \in [0,T]$ wit $\theta \in [0,h/m]$. by using Taylor's theorem.\\
	Let $m\geq m_0$ be fixed. We will show by induction that
	\begin{align}
		\label{eq:ivWtilde}
		\norm{G_{i,m}-\Phi_{i,m}}_{C([0,T]\times W;\R^{n})}\leq \left(1+\frac{C_\Phi}{m}\right)^i \frac{iR}{m \exp(C_\Phi)}\leq R.
	\end{align}
	This is obvious for $i=1$. So going from $i$ to $i+1$ we have
	\begin{align*}
		\norm{G_{i+1,m}- \Phi_{i+1,m}}_{C([0,T]\times W;\R^{n})}
		\leq&\max_{h\in [0,T]} \normC{g_{{i+1,m}}(h, G_{i,m}(h,\cdot)) - \Phi(h/m,G_{i,m}(h,\cdot))}\\
		&+ \max_{h\in[0,T]}\normC{\Phi(h/m,G_{i,m}(h,\cdot)) - \Phi(h/m, \Phi(ih/m,\cdot))}\\
		\leq& \frac{C}{m^2} + \left(1 + \frac{C_\Phi}{m}\right)\norm{G_{i,m} - \Phi_{i,m}}_{C([0,T]\times W;\R^n)}\\
		\leq& \frac{R}{m \exp(C_\Phi)} + \left(1 + \frac{C_\Phi}{m}\right)^{i+1} \frac{iR}{m \exp(C_\Phi)}\\
		\leq& \left(1+ \frac{C_\Phi}{m}\right)^{i+1} \frac{(i+1)R}{m \exp(C_\Phi)}\leq R,
	\end{align*}
	where we used the monotone convergence of $(1+C_\Phi/m)^m\to \exp(C_\Phi)$ and apply \cref{eqn:LPhi}. Therefore, for all $m\geq m_0, i\leq m$ and $(h,x) \in [0,T]\times W$ we have $G_{i,m}(h,x)\in \tilde{W}$. 
\end{proof}

\begin{lem}
	\label{lem:lem1}
	Let $n, r \in \N, ~W, \tilde{W} \subseteq \R^{n}$ be compact sets, $f \in C^{r}(\tilde{W}),~ g_1, g_2 \in C^r(W;\tilde{W})$ and $\alpha \in \N_0^n,~ \abs{\alpha} \leq r$. Furthermore, let $D^\alpha f$ be Lipschitz continuous on $\tilde{W}$ for $\abs{\alpha}=r$. Latsly, let there be a compact set $W_P$ such that
	\begin{align*}
		\left[\bigtimes_{\abs{\gamma}=0}^{r} \bigtimes_{i=1}^{n}\right] \left[D^\gamma g_1^{i}(W)\right] \cup\left[\bigtimes_{\abs{\gamma}=0}^{r} \bigtimes_{i=1}^{n}\right] \left[D^\gamma g_2^{i}(W)\right]\subseteq W_P.
	\end{align*}
	Then there is a constant $C_{\tilde{W},W_P,\alpha,f} <\infty$ only depending on $\tilde{W}, \alpha, f$ and $W_P$ such that
	\begin{align}
		\label{eqn:lem1}
		\sup_{x \in W} \abs{D^\alpha\left(f(g_1(x))- f(g_2(x)) \right)} \leq
		C_{\tilde{W},W_P,\alpha,f} \norm{g_1 -g_2}_{C^r(W; R^n)}.
	\end{align}
\end{lem}
\begin{proof}
	Let $x \in W$. With \cref{eq:cor} we have
	\begin{align}
        \begin{aligned}
		\label{eqn:triangular}
		&\abs{D^\alpha(f \circ g_1 - f \circ g_2)(x)} \\=&
		\abs{\sum_{\abs{\beta}=0}^{\abs{\alpha}}\left(D^\beta f\right)(g_1(x))P_{\alpha, \beta, n}\left(\left[D^\gamma g_1^{(i)}(x)\right]_{\gamma,i}\right)-
			\left(D^\beta f\right)(g_2(x))P_{\alpha, \beta, n}\left(\left[D^\gamma g_2^{(i)}(x)\right]_{\gamma,i}\right)} \\
		\leq& \sum_{\abs{\beta}=0}^{\abs{\alpha}}
		\abs{\left(D^\beta f\right)(g_1(x))P_{\alpha, \beta, n}\left(\left[D^\gamma g_1^{(i)}(x)\right]_{\gamma,i}\right) - \left(D^\beta f\right)(g_2(x))P_{\alpha, \beta, n}\left(\left[D^\gamma g_1^{(i)}(x)\right]_{\gamma,i}\right)} \\
		&+\abs{\left(D^\beta f\right)(g_2(x))P_{\alpha, \beta, n}\left(\left[D^\gamma g_1^{(i)}(x)\right]_{\gamma,i}\right) - \left(D^\beta f\right)(g_2(x))P_{\alpha, \beta, n}\left(\left[D^\gamma g_2^{(i)}(x)\right]_{\gamma,i}\right)} \\
		=& \sum_{\abs{\beta}=0}^{\abs{\alpha}}
		\abs{P_{\alpha, \beta, n}\left(\left[D^\gamma g_1^{(i)}(x)\right]_{\gamma,i}\right)}\cdot
		\abs{\left(D^\beta f\right)(g_1(x))-\left(D^\beta f\right)(g_2(x))} \\
		&+\abs{\left(D^\beta f\right)(g_2(x))}\cdot
		\abs{P_{\alpha, \beta, n}\left(\left[D^\gamma g_1^{(i)}(x)\right]_{\gamma,i}\right)-P_{\alpha, \beta, n}\left(\left[D^\gamma g_2^{(i)}(x)\right]_{\gamma,i}\right)}
        \end{aligned}
	\end{align}
	Since $g_j \in C^r(W;\tilde{W}),~ j=1,2$ is $r$-times continuously differentiable, $D^\gamma g_j(W)$ is compact for all $\abs{\gamma} \leq r$. Therefore, $W_P$
	is compact as well and $P_{\alpha, \beta, n}$ is Lipschitz continuous on $W_P$. Hence, there is a constant $L_{P_{\alpha, \beta, n}, W_P}<\infty$ only depending on $W_P, \alpha, \beta$ and $n$. such that
	\begin{align}
        \begin{aligned}
		\label{eqn:lipPol}
		&\abs{P_{\alpha, \beta, n}\left(\underbrace{\left[D^\gamma g_1^{(i)}(x)\right]_{\gamma,i}}_{\in W_P}\right)- P_{\alpha, \beta, n}\left(\underbrace{\left[D^\gamma g_2^{(i)}(x)\right]_{\gamma,i}}_{\in W_P}\right)}\\
		\leq& L_{P_{\alpha, \beta, n},W_P} \max_{0 \leq\abs{\gamma}\leq r} \max_{1 \leq i \leq n}\abs{D^\gamma g_1^{(i)}(x)-D^\gamma g_2^{(i)}(x)}\\
		\leq& L_{P_{\alpha, \beta, n}, W_P} \normCr{g_1-g_2}.
        \end{aligned}
	\end{align}
	Furthermore, the image of $P_{\alpha, \beta, n}$ is bounded on $W_P$. and we get
	\begin{align}
		\label{eqn:CPol}
		\abs{P_{\alpha, \beta, n}\left(\left[D^\gamma g_1^{(i)}(x)\right]_{\gamma,i}\right)}\leq
		\sup_{x_P \in W_P} \abs{P_{\alpha, \beta, n}(x_P)}=:\hat{C}_{P_{\alpha, \beta, n}, W_P}<\infty.
	\end{align}
	Let
	\begin{align}\label{eqn:Cf}
		\hat{C}_{\tilde{W},f,\beta}:=  \sup_{\tilde{x} \in \tilde{W}}\abs{D^\beta f(\tilde{x})}.
	\end{align}
	Since $f \in C^r(\tilde{W})$ is $r$-times continuously differentiable on $\tilde{W}$.~$\hat{C}_{\tilde{W}f,\beta}<\infty$ exists for $\abs{\beta} \leq r$.
	On the other hand, there are Lipschitz constants $L_{\tilde{W},f,\beta} < \infty$ such that
	\begin{align}
		\label{eqn:lipf}
		\abs{\left(D^\beta f\right)(g_1(x))-\left(D^\beta f\right)(g_2(x))}\leq L_{\tilde{W}, f,\beta}
		\max_{1 \leq i \leq n} \abs{g_1^{(i)}(x)-g_2^{(i)}(x)} \leq L_{\tilde{W},f, \beta} \normCr{g_1 - g_2}
	\end{align}
	for all $x \in W, \abs{\beta}\leq r$.
	By combining \cref{eqn:lipPol,eqn:CPol,eqn:Cf,eqn:lipf} with \cref{eqn:triangular} we get
	\begin{align*}
		&\abs{D^\alpha(f \circ g_1 - f \circ g_2)(x)}\\
		\leq& \sum_{\abs{\beta}=0}^{\abs{\alpha}}
		\abs{P_{\alpha, \beta, n}\left(\left[D^\gamma g_1^{(i)}(x)\right]_{\gamma,i}\right)}\cdot
		\abs{\left(D^\beta f\right)(g_1(x))-\left(D^\beta f\right)(g_2(x))}\\
		&+\abs{\left(D^\beta f\right)(g_2(x))}\cdot
		\abs{P_{\alpha, \beta, n}\left(\left[D^\gamma g_1^{(i)}(x)\right]_{\gamma,i}\right)-P_{\alpha, \beta, n}\left(\left[D^\gamma g_2^{(i)}(x)\right]_{\gamma,i}\right)}\\
		\leq&\sum_{\abs{\beta}=0}^{\abs{\alpha}}\left(\hat{C}_{P_{\alpha, \beta, n},W_P} \cdot L_{\tilde{W}, f, \beta}+ \hat{C}_{\tilde{W}, f,\beta} \cdot L_{P_{\alpha, \beta, n},W_P}\right)\normCr{g_1 -g_2}
	\end{align*}
	for all $x \in W$ and therefore for the supremum in particular. So \cref{eqn:lem1} holds true with
	\begin{align}\label{eqn:createC}
		C_{\tilde{W},W_P,\alpha,f}:=\sum_{\abs{\beta}=0}^{\abs{\alpha}}\left(\hat{C}_{P_{\alpha, \beta, n},W_P} \cdot L_{\tilde{W}, f, \beta}+ \hat{C}_{\tilde{W}, f,\beta} \cdot L_{P_{\alpha, \beta, n},W_P}\right).
	\end{align}
\end{proof}

\begin{lem}
	\label{lem:lem2}
	Let $n, r \in \N,~W, \tilde{W} \subseteq \R^{n}$ be compact sets, $f_1,f_2 \in C^{r}(\tilde{W}),~ g \in C^r(W;\tilde{W})$ and $\alpha \in \N_0^n,~ \abs{\alpha} \leq r$. Furthermore, let there be a compact set $\tilde{W}_P$ such that
	\begin{align*}
		\left[\bigtimes_{\abs{\gamma}=0}^{r} \bigtimes_{i=1}^{n}\right] \left[D^\gamma g^{(i)}(W)\right]\subseteq \tilde{W}_P.
	\end{align*}
	Then there is a constant $\bar{C}_{\tilde{W}_P, \alpha}<\infty$ only depending on $\alpha$ and $\tilde{W}_P$ such that
	\begin{align}
		\label{eqn:lem2}
		\sup_{x \in W} \abs{D^\alpha\left(f_1(g(x))-f_2(g(x))\right)}\leq \bar{C}_{\tilde{W}_P, \alpha}\norm{f_1-f_2}_{C^r(\tilde{W})},
	\end{align}
\end{lem}
\begin{proof}
	Let $x \in W$. With \cref{eq:cor} we have
	\begin{align}
		\label{eqn:triangular2}
		&\abs{D^\alpha\left(f_1(g(x))-f_2(g(x))\right)}\notag\\
		=&\abs{\sum_{\abs{\beta}=0}^{\abs{\alpha}}\left(D^\beta f_1\right)(g(x))P_{\alpha, \beta, n}\left(\left[D^\gamma g^{(i)}(x)\right]_{\gamma,i}\right)-
			\left(D^\beta f_2\right)(g(x))P_{\alpha, \beta, n}\left(\left[D^\gamma g^{(i)}(x)\right]_{\gamma,i}\right)}\\
		\leq&\sum_{\abs{\beta}=0}^{\abs{\alpha}}\abs{P_{\alpha, \beta, n}\left(\left[D^\gamma g^{(i)}(x)\right]_{\gamma,i}\right)}\cdot
		\abs{\left(D^\beta f_1\right)(g(x))-\left(D^\beta f_2\right)(g(x))}. \notag
	\end{align}
	Here we recall \cref{eqn:CPol} to bound the first factor of the right-hand side of \cref{eqn:triangular2} from above. To handle the second factor, we use $y:=g(x)\in \tilde{W}$ and get
	\begin{align*}
		\abs{D^\alpha\left(f_1(g(x))-f_2(g(x))\right)}&\leq \sum_{\abs{\beta}=0}^{\abs{\alpha}}\hat{C}_{P_{\alpha, \beta, n},\bar{W}_P}\sup_{y \in \tilde{W}} \abs{D^\beta f_1(y)-D^\beta f_2(y)}\\
		&\leq \bar{C}_{W,\tilde{W}_P, \alpha} \norm{f_1-f_2}_{C^r(\tilde{W})}	
	\end{align*}
	with $\bar{C}_{\tilde{W}_P, \alpha}:=\max_{0 \leq\abs{\beta}\leq \abs{\alpha}} \hat{C}_{P_{\alpha, \beta, n},\tilde{W}_P}$.
\end{proof}

\begin{lem}
	\label{lem:Wp}
	Let $n, m, r\in \N,~ U \subseteq \R^n$ open, $f \in C^r(U;\R^n)$ and $\Phi:\R^+_0 \times U \to \R^n$ be the phase flow of the ODE
	\begin{align*}
		\dot{x}(t)=f(x), \qquad x(0)=x_0,
	\end{align*}
	i.e.\ $\Phi(t,x_0)=x(t)$. Let $T>0$ such that $\Phi(t,x)$ exists for all $t \in [0,T],~x \in U$ and $\Phi([0,T]\times U) \subseteq U$. We define $\Phi_{i,m}(h,x):=\Phi(ih/m,x)$. Furthermore, let $g_i \in C^r([0,T]\times\R^n;\R^n)$, $g_{i,m}(h,x):=g_i(h/m,x)$ and $G_{i,m}(h,x):=g_{i,m}(h,g_{i-1,m}(h,...g_{1,m}(h,x)...))$ for $i=1,..,m$. For a compact set $W \subseteq U$ we define $W_\Phi:=\Phi([0,T],W)$ and $\tilde{W}:=\overline{B_R(W_\Phi)}$. with $R <\dist(W_\Phi,\partial U)$.
	Let there be a constant $C$ independent of $m$ such that such that
	\begin{align}
		\label{eq:CH}
		\norm{g_{i,m}-\Phi_{1,m}}_{C^r_\text{max}([0,T]\times \tilde{W};\R^n)}\leq \frac{C}{m^2}.
	\end{align}
	If we define
	\begin{align*}
		W_P:=\overline{B_1\left(\left[\bigtimes_{\abs{\gamma}=0}^{r} \bigtimes_{j=1}^{n}\right]\left[D^\gamma \Phi^{(j)}([0,T],\tilde{W})\right]\right)},
	\end{align*}
	there is $m_0 \in \N$ such that
	\begin{align}
		\label{eq:WP}
		\left[\bigtimes_{\abs{\gamma}=0}^{r} \bigtimes_{j=1}^{n}\right]D^\alpha G_{i,m}^{(j)}(W)\subseteq W_P
	\end{align}
	for all $m \geq m_0,~ i\leq m,~ \abs{\alpha}\leq r$.
\end{lem}
\begin{proof}
    First we show that under the assumptions of this lemma there is a constant $C_\Phi$ such that for all $D^\beta = \partial_{x_j}$
	\begin{align}
		\label{eq:beta=1}
		\sup_{(h,\tilde{x})\in [0,T]\times \tilde{W}}\abs{D^\beta g_{i,m}^{(k)}(h,\tilde{x})}\leq \delta_{k,j} + \frac{C_{\Phi}}{m}, 
	\end{align}
	and for all $1<\abs{\beta}<r$ or $D^\beta=\partial_t$
	\begin{align}
		\label{eq:beta>1}
		\sup_{(h,\tilde{x})\in [0,T]\times\tilde{W}}\abs{D^\beta g_{i,m}^{(k)}(h,\tilde{x})}\leq \frac{C_{\Phi}}{m}.
	\end{align}
    To do that, we observe with
	\begin{align*}
		\partial_t\Phi(t,x)=\dot{x}(t)=f(x(t))=f(\Phi(t,x))
	\end{align*}
	that $\Phi$ is $r+1$-times continuously differentiable as long as one derivative is a time derivative. Let $h \in [0,T]$ and $x \in \tilde{W}$. Let $i=1,...,m$. The triangle inequality and \cref{eq:CH} give
	\begin{align*}
		\abs{D^\beta g_{i,m}^{(k)}(h,x)}\leq \abs{D^\beta\left(g_{i,m}^{(k)}(h,x)-\Phi^{(k)}_{1,m}(h,x)\right)}+\abs{D^\beta\Phi_{1,m}^{(k)}(h,x)}\leq \frac{C}{m^2}+ \abs{D^\beta\Phi_{1,m}^{(k)}(h,x)}.
	\end{align*}
	In case we have $D^\beta=\partial_{x_j}$ for some $j=1,,..,n$. we use the Taylor-expansion of $\Phi$ with respect to time, to get
	\begin{align*}
		\abs{D^\beta g_{i,m}^{(k)}(h,x)}&\leq \frac{C}{m^2} + \abs{D^\beta\Phi_{1,m}^{(k)}(h,x)} \\
		&\leq \frac{C}{m^2}+\abs{\partial_{x_j}\left(x_k+\frac{h}{m}\partial_t\Phi(\theta,x)\right)}\\
		&\leq \frac{C}{m} + \delta_{k,j} + \frac{T}{m} \norm{\Phi}_{C^r_\text{max}([0,T]\times \tilde{W};\R^n)},
	\end{align*}
	with $\theta \in (0,h/m)$. Therefore, \cref{eq:beta=1} holds with $C_{\Phi}=C+ \max\{1,T\}\norm{\Phi}_{C^r_\text{max}([0,T]\times \tilde{W};\R^n)}$.\\
	Now we show \cref{eq:beta>1}. If $D^\beta=\partial_t$. we have
	\begin{align*}
		\abs{D^\beta g_{i,m}^{(k)}(h,x)}\leq \frac{C}{m^2}+ \abs{\frac{1}{m}\partial_t \Phi(h/m,x)}\leq \frac{C}{m}+ \frac{1}{m}\norm{\Phi}_{C^r_\text{max}([0,T]\times \tilde{W};\R^n)}\leq \frac{C_\Phi}{m}.
	\end{align*}
	The same is true for any higher derivative, that includes a time derivative. The only case left is $D^\beta$ only including spatial derivatives. Using Taylor's theorem again, we get
	\begin{align*}
		\abs{D^\beta g_{i,m}^{(k)}(h,x)}\leq \frac{C}{m^2}+\abs{D_x^\beta\left(x_k+\frac{h}{m}\partial_t\Phi(\theta,x)\right)}\leq \frac{C}{m^2}+\frac{T}{m}\norm{\Phi}_{C^r_\text{max}([0,T]\times \tilde{W};\R^n)}\leq\frac{C_\Phi}{m}.
	\end{align*}
	Finally, we show \cref{eq:WP}. For $m\geq m_0$ and $i<m$ we will prove \cref{eq:WP} by induction over $i$. We need to show 
	\begin{align*}
		\normtmax{G_{i+1,m}-\Phi_{i+1,m}} \leq 1.
	\end{align*}
	To be able to use \Cref{the:faacor}, we need to extend $\Phi,g_{i,m}, \Phi_{i,m}$ and $G_{i,m}$ by a time output, i.e. 
	\begin{align*}
		\tilde{f}(h,x):=\begin{pmatrix}
			h\\ f(h,x)
		\end{pmatrix}, \qquad f=\Phi,g_{i,m}, \Phi_{i,m}, G_{i,m},
	\end{align*}
	To do this, we observe with the triangle inequality
	\begin{align}
        \begin{aligned}
		\label{eq:wp1}
		&\normtmax{G_{i+1,m}-\Phi_{i+1,m}} \\
		=&\norm{\tilde{G}_{i+1,m}-\tilde{\Phi}_{i+1,m}}_{C_\text{max}^r([0,T]\times W;\R^{n+1})}\\
		\leq&
		\norm{\tilde{g}_{i+1,m} \circ \tilde{G}_{i,m} - \tilde{g}_{i+1,m} \circ \tilde{\Phi}_{i,m}}_{C_\text{max}^r([0,T]\times W;\R^{n+1})}\\ 
		&+
		\norm{\tilde{g}_{i+1,m} \circ \tilde{\Phi}_{i,m} - \tilde{\Phi}_{1,m} \circ \tilde{\Phi}_{i,m}}_{C_\text{max}^r([0,T]\times W;\R^{n+1})}.
        \end{aligned}
	\end{align}
	Now we estimate both summands separately.\\
    To bound $\norm{\tilde{g}_{i+1,m} \circ \tilde{G}_{i,m} - \tilde{g}_{i+1,m} \circ \tilde{\Phi}_{i,m}}_{C_\text{max}^r([0,T]\times W;\R^{n+1})}$, we observe with \Cref{the:faacor}
	\begin{align*}
		&\abs{D^\alpha \left( \tilde{g}_{i+1,m}^{(k)}\circ \tilde{G}_{i,m} - \tilde{g}_{i+1,m}^{(k)} \circ \tilde{\Phi}_{i,m}\right)(h,x)}\\
		=&\left\vert\sum_{\abs{\beta}=1}^{\abs{\alpha}}(D^\beta \tilde{g}_{i+1,m}^{(k)})(\tilde{G}_{i,m}(h,x))P_{\alpha,\beta,n+1}\left(\left[D^\gamma \tilde{G}_{i,m}^{(l)}(h,x)\right]_{\gamma,l}\right)\right.\\
		&\qquad\left.
		-(D^\beta \tilde{g}_{i+1,m}^{(k)})(\tilde{\Phi}_{i,m}(h,x))P_{\alpha,\beta,n+1}\left(\left[D^\gamma \tilde{\Phi}^{(l)}_{i,m(h,x)}\right]_{\gamma,l}\right)\right\vert\\
		\leq&\sum_{\abs{\beta}=1}^{\abs{\alpha}}\left\vert(D^\beta \tilde{g}_{i+1,m}^{(k)})(\tilde{G}_{i,m}(h,x))P_{\alpha,\beta,n+1}\left(\left[D^\gamma \tilde{\Phi}_{i,m}^{(l)}(h, x)\right]_{\gamma,l}\right)\right.\\
		&\qquad\left.-(D^\beta \tilde{g}_{i+1,m}^{(k)})(\tilde{\Phi}_{i,m}(h,x))P_{\alpha,\beta,n+1}\left(\left[D^\gamma \tilde{\Phi}_{i,m}^{(l)}(h,x)\right]_{\gamma,l}\right)\right\vert\\
		&+\left\vert(D^\beta \tilde{g}_{i+1,m}^{(k)})(\tilde{G}_{i,m}(h,x))P_{\alpha,\beta,n+1}\left(\left[D^\gamma \tilde{\Phi}_{i,m}^{(l)}(h, x)\right]_{\gamma,l}\right)\right.\\
		&\qquad\left.-(D^\beta \tilde{g}_{i+1,m}^{(k)})(\tilde{G}_{i,m}(h,x))P_{\alpha,\beta,n+1}\left(\left[D^\gamma \tilde{G}_{i,m}(h,x)\right]_{\gamma,l}\right)\right\vert\\
		\leq&
		\sum_{\abs{\beta}=1}^{\abs{\alpha}} \abs{P_{\alpha, \beta, n+1 }\left(\left[D^\gamma \tilde{\Phi}_{i,m}^{(l)}(h, x)\right]_{\gamma,l}\right)}\cdot
		\abs{(D^\beta \tilde{g}_{i+1,m})(\tilde{\Phi}_{i,m}(h,x))-(D^\beta \tilde{g}_{{i,m}})(\tilde{G}_{i,m}(h,x))}\\
		&+\abs{(D^\beta \tilde{g}^{(k)}_{{i,m}})(\tilde{G}_{i,m}(h,x))}\cdot \abs{P_{\alpha, \beta, n+1 }\left(\left[D^\gamma \tilde{\Phi}_{i,m}^{(l)}(h, x)\right]_{\gamma,l}\right)-P_{\alpha, \beta, n+1 }\left(\left[D^\gamma \tilde{G}_{i,m}^{(l)}(h,x)\right]_{\gamma,l}\right)}.
	\end{align*}
	For $\abs{\beta}=1$ with $\beta=e_j$. with \cref{eq:beta=1} and \Cref{lem:beta=1} we have
	\begin{align}
        \begin{aligned}            
		\label{eq:wp:beta=1}
		&\sum_{\abs{\beta}=1}^1 \abs{(D^\beta \tilde{g}^{(k)}_{i+1,m})(\tilde{G}_{i,m}(h,x))} \cdot
		\abs{P_{\alpha, \beta, n+1 }\left(\left[D^\gamma \tilde{\Phi}_{i,m}^{(l)}(h, x)\right]_{\gamma,l}\right)-P_{\alpha, \beta, n+1 }\left(\left[D^\gamma \tilde{G}_{i,m}^{(l)}(h,x)\right]_{\gamma,l}\right)}\\
		\leq&
		\sum_{j=1}^{n+1} \left(\delta_{k,j}+\frac{C_{\Phi}}{m}\right)\norm{\tilde{\Phi}_{i,m}-\tilde{G}_{i,m}}_{C_\text{max}^r([0,T]\times W;\R^{n+1})}\\
		\leq& \left(1+\frac{C_{\Phi}(n+1)}{m}\right)\norm{\tilde{\Phi}_{i,m}-\tilde{G}_{i,m}}_{C_\text{max}^r([0,T]\times W;\R^{n+1})}.
        \end{aligned}
	\end{align}
	For $\abs{\beta}>1$. we define
	\begin{align*}
		C_{\alpha,\beta}:= \sum_{s=1}^{\abs{\alpha}} \sum_{p_s(\alpha,\beta)}(\alpha!) (2^s-1) \max\left\{\max_{t \in [0,T]}\norm{\tilde{\Phi}}_{C_\text{max}^r([0,T]\times W;\R^{n+1})},1\right\}^{s-1},
	\end{align*}
	and observe with \Cref{the:faa}
	\begin{align}
        \begin{aligned}
		\label{eq:wp:beta>1}
		&\abs{P_{\alpha,\beta, n+1}\left(\left[D^\gamma \tilde{\Phi}^{(l)}(h, x)\right]_{\gamma,l}\right)-P_{\alpha, \beta, n+1 }\left(\left[D^\gamma \tilde{G}_{i,m}^{(l)}(h,x)\right]_{\gamma,l}\right)}\\
		=&\abs{\sum_{s=1}^{\abs{\alpha}}\sum_{p_s(\alpha,\beta)}(\alpha !)\left[\prod_{j=1}^s \frac{[D^{l_j}\tilde{\Phi}_{i,m}^{(l)}(h,x)]^{k_j}}{(k_j!)(l_j!)^{\abs{k_j}}}- \prod_{j=1}^s \frac{[D^{l_j}\tilde{G}_{i,m}^{(l)}(h,x)]^{k_j}}{(k_j!)(l_j!)^{\abs{k_j}}}\right] }\\
		\leq& \sum_{s=1}^{\abs{\alpha}}\sum_{p_s(\alpha,\beta)}(\alpha !) (2^s-1) \max\left\{\norm{\tilde{\Phi}_{i,m}}_{C_\text{max}^r([0,T]\times W;\R^{n+1})},\norm{\tilde{\Phi}_{i,m}-\tilde{G}_{i,m}}_{C_\text{max}^r([0,T]\times W;\R^{n+1})}\right\}^{s-1}\\ 
		&\qquad \qquad ~\cdot\norm{\tilde{\Phi}_{i,m}-\tilde{G}_{i,m}}_{C_\text{max}^r([0,T]\times W;\R^{n+1})}\\
		\leq& C_{\alpha,\beta} \norm{\tilde{\Phi}_{i,m}-\tilde{G}_{i,m}}_{C_\text{max}^r([0,T]\times W;\R^{n+1})},
        \end{aligned}
	\end{align}
	where we used
	\begin{align*}
		\abs{\prod_{i=1}^s a_i - \prod_{i=1}^s b_i}&\leq (2^s -1) \max \left\{\max_{i=1,..,s} \abs{b_i}, \max_{i=1,..,s}\abs{a_i-b_i}\right\}^{s-1} \abs{a_i-b_i},
	\end{align*}
	and $(k_j!)(l_j!)^{\abs{k_j}}\geq 1$ for $s=1,...,\abs{\alpha},~j=1,..,s$ and assumed
	\begin{align*}
		\norm{\tilde{\Phi}_{i,m}-\tilde{G}_{i,m}}_{C_\text{max}^r([0,T]\times W;\R^{n+1})}=\norm{\Phi_{i,m}-G_{i,m}}_{C_\text{max}^r([0,T]\times W;\R^{n+1})}\leq 1.
	\end{align*}
	So the latter assumption will have to be include in our induction hypotheses \cref{eq:iv}.
	Since $G_{i,m}(h,x) \in \tilde{W}$ for $h \in [0,T],~x \in W$. we get
	\begin{align}
		\label{eq:wp2}
		&\sum_{\abs{\beta}=1}^{\abs{\alpha}} \abs{(D^\beta \tilde{g}^{(k)}_{i+1,m})(\tilde{G}_{i,m}(h,x))} \cdot
		\abs{P_{\alpha, \beta, n+1 }\left(\left[D^\gamma \tilde{\Phi}_{i,m}^{(l)}(h, x)\right]_{\gamma,l}\right)-P_{\alpha, \beta, n+1 }\left(\left[D^\gamma \tilde{G}_{i,m}^{(l)}(x)\right]_{\gamma,l}\right)}\notag \\
		\leq& \left(1+ \frac{C_{\alpha,\Phi}}{m}\right)\norm{\tilde{\Phi}_{i,m}-\tilde{G}_{i,m}}_{C_\text{max}^r([0,T]\times W;\R^{n+1})}
	\end{align}
	using \cref{eq:beta>1,eq:wp:beta=1,eq:wp:beta>1} with
	\begin{align*}
		C_{\alpha,\Phi}:=\max \left\{\max_{1<\abs{\beta}\leq \abs{\alpha}} \{C_{\alpha,\beta}\}, (n+1)\right\}C_\Phi.
	\end{align*}
	To bound the second summand of \cref{eq:wp1}, i.e., $\norm{\tilde{g}_{i+1,m} \circ \tilde{\Phi}_{i,m} - \tilde{\Phi}_{1,m} \circ \tilde{\Phi}_{i,m}}_{C_\text{max}^r([0,T]\times W;\R^{n+1})}$ we observe
	\begin{align}
        \begin{aligned}
		\label{eq:wp3}
		&\abs{D^\alpha \left((\tilde{g}_{i+1,m} \circ \tilde{\Phi}_{i,m})(h,x) - (\tilde{\Phi}_{1,m}\circ \tilde{\Phi}_{i,m})(h,x)\right)}  \\
		=&\left\vert \sum_{\abs{\beta}=1}^{\abs{\alpha}}(D^\beta \tilde{g}_{i+1,m})(\tilde{\Phi}_{i,m}(h,x)) P_{\alpha,\beta,n+1}\left(\left[D^\gamma\tilde{\Phi}_{i,m}^{(l)}(h, x)\right]_{\gamma,l}\right)\right.  \\
		&\left.\qquad -(D^\beta \tilde{\Phi}_{1,m})(\tilde{\Phi}_{i,m}(h,x)) P_{\alpha,\beta,n+1}\left(\left[D^\gamma\tilde{\Phi}_{i,m}^{(l)}(h, x)\right]_{\gamma,l}\right)\right\vert   \\
		\leq &
		\sum_{\abs{\beta}=1}^{\abs{\alpha}}\abs{P_{\alpha, \beta, n+1 }\left(\left[D^\gamma \tilde{\Phi}_{i,m}^{(l)}(h, x)\right]_{\gamma,l}\right)}\cdot
		\abs{(D^\beta \tilde{g}_{i+1,m})(\tilde{\Phi}_{i,m}(h,x))-(D^\beta \tilde{\Phi}_{1,m})(\tilde{\Phi}_{i,m}(h,x))}  \\
		\leq &
		C_{\alpha, W_P} \norm{\tilde{g}_{i+1,m}-\tilde{\Phi}_{1,m}}_{C_\text{max}^r([0,T]\times W;\R^{n+1})}
		\leq \frac{C_{\alpha, W_P} C}{m^2},
        \end{aligned}
	\end{align}
	with
	\begin{align*}
		C_{\alpha,W_P}&:=\max_{\abs{\beta}\leq \abs{\alpha}} \sup_{(h,x_P) \in [0,T]\times W_P} \abs{P_{\alpha,\beta, n+1}(h,x_P)}
	\end{align*}
	Now let $m_0:=\lceil C_{r, W_P} C \exp(C_{r,\Phi}) \rceil$. with $C_{r,W_P}:=\max_{\abs{\alpha}\leq r}C_{\alpha,W_P}$ and $C_{r,\Phi}:=\max_{\abs{\alpha}\leq r}C_{\alpha,\Phi}$. We will prove by induction that for all $m\geq m_0, i\leq m$
	\begin{align}
		\label{eq:iv}
		\norm{G_{i,m}-\Phi_{i,m}}_{C_\text{max}^r([0,T]\times W;\R^{n})}\leq \left(1 + \frac{C_{r,\Phi}}{m}\right)^i\frac{i}{m \exp({C_{r,\Phi}})}\leq 1
	\end{align}
	holds. For $i=1$. we have
	\begin{align*}
		\normtmax{G_{i,m}-\Phi_{i,m}}&\leq\norm{g_{{1,m}}-\Phi_{1,m}}_{C_{\max}^r([0,T]\times\tilde{W};\R^n)}\leq \frac{C}{m^2}\\
		&\leq \left( 1+ \frac{C_{r,\Phi}}{m}\right)^1 \frac{1}{m\exp(C_{r,\Phi})}\leq 1.
	\end{align*}
	Going from $i$ to $i+1$. we observe with $m \geq m_0$.~\cref{eq:wp1,eq:wp2,eq:wp3}
	\begin{align*}
		&\normCrmax{G_{i+1}-\Phi((i+1)h/m,\cdot)}\\
		\leq&
		\norm{\tilde{g}_{i+1,m} \circ \tilde{G}_{i,m} - \tilde{g}_{i+1,m} \circ \tilde{\Phi}_{i,m}}_{C_\text{max}^r([0,T]\times W;\R^{n+1})}\\
		&+\norm{\tilde{g}_{i+1,m} \circ \tilde{\Phi}_{i,m} - \tilde{\Phi}_{1,m} \circ \tilde{\Phi}}_{C_\text{max}^r([0,T]\times W;\R^{n+1})}\\
		\leq& 
		\max_{\abs{\alpha}\leq r}\max_{1\leq k \leq n}\sup_{(h,x)\in [0,T]\times W}\abs{D^\alpha \left( \tilde{g}_{i+1,m}^{(k)}\circ \tilde{G}_{i,m} - \tilde{g}_{i+1,m}^{(k)} \circ \tilde{\Phi}_{i,m}\right)(h,x)}\\
		&+\max_{\abs{\alpha}\leq r}\max_{1\leq k \leq n}\sup_{(h,x)\in [0,T]\times W}\abs{D^\alpha \left((\tilde{g}_{i+1,m} \circ \tilde{\Phi}_{i,m})(h,x) - (\tilde{\Phi}_{1,m}\circ \tilde{\Phi}_{i,m})(h,x)\right)}\\
		\leq&\max_{\abs{\alpha}\leq r} \left(1+\frac{C_{\alpha,\Phi}}{m}\right)\norm{\tilde{\Phi}-\tilde{G}_{i,m}}_{C_\text{max}^r([0,T]\times W;\R^{n+1})}
		+\max_{\abs{\alpha}\leq r}\frac{C_{\alpha,W_P}C}{m^2}\\
		\leq&\left(1+\frac{C_{r,\Phi}}{m}\right)\normtmax{\Phi_{i,m}-G_{i,m}}+
		\frac{1}{m \exp(C_{r,\Phi})}\\
		\leq& \left(1+ \frac{C_{r,\Phi}}{m}\right)^{i+1}\frac{i+1}{m \exp(C_{r,\Phi})}\leq 1.
	\end{align*}
	Therefore, \cref{eq:WP} holds for all $m\geq m_0$.
\end{proof}
Note that we could choose any other finite radius to construct $\tilde{W}_P$. We chose $1$ to simplify calculations. Nevertheless, $m_0$ obviously depends on this choice. So if one wants to use \Cref{lem:Wp} in isolation there is a tradeoff between how small one can get $m_0$ and how small we choose the radius.\\
Now we are prepared to correct \cref{eqn:turaev_false}, by introducing additional assumptions, which are fulfilled in the symplectic polynomial approximation theorem anyway.
\begin{theorem}
	\label{the:comp}
	Let $n, m, r\in \N,~ U \subseteq \R^n$ open, $f \in C^r(U;\R^n)$ and $\Phi:\R^+_0 \times U \to \R^n$ be the phase flow of the ODE
	\begin{align*}
		\dot{x}(t)=f(x), \qquad x(0)=x_0,
	\end{align*}
	i.e.\ $\Phi(t,x_0)=x(t)$. Let $T>0$ such that $\Phi(t,x)$ exists for all $t \in [0,T],~x \in U$ and $\Phi([0,T]\times U) \subseteq U$. We define $\Phi_{i,m}(h,x):=\Phi(ih/m,x)$. Furthermore, let $g_i \in C^r([0,T]\times\R^n;\R^n)$, $g_{i,m}(h,x):=g_i(h/m,x)$ and $G_{i,m}(h,x):=g_{i,m}(h,g_{i-1,m}(h,...g_{1,m}(h,x)...))$ for $i=1,...,m$. For a compact set $W \subseteq U$ we define $W_\Phi:=\Phi([0,T],W)$ and $\tilde{W}:=\overline{B_R(W_\Phi)}$. with $R < \dist(W_\Phi,\partial U)$.
	Let there be constant $C<\infty$ independent of $m$ such that
	\begin{align}
		\label{eqn:compcond}
		\norm{g_{i,m}-\Phi_{1,m}}_{C^r([0,T]\times \tilde{W};\R^n)}\leq \frac{C}{m^2}
	\end{align}
	for $i=1,...,m$. Then there is $m_0 \in \N$ such that for all $m \geq m_0$ the difference between the composition $G_{m,m}$ and the phase flow $\Phi$ decreases with $1/m$, i.e.,
	\begin{align}
		\label{eq:comp}
		\normt{\Phi-G_{m,m}}\leq \frac{\tilde{C}}{m},
	\end{align}
	where $\tilde{C}$ does not depend on $m$.
\end{theorem}
\begin{proof}
	For this proof we will use the same notation as used in \Cref{lem:Wp}, i.e., $\Phi_{i,m},~G_{i,m},~g_{i,m},~\tilde{\Phi}_{i,m}$, $\tilde{G}_{i,m},~\tilde{g}_{i,m}$.
	Let $m_0 \in \N$ such that it satisfies the conditions of \Cref{lem:Wtilde} and \Cref{lem:Wp}.
	With \cref{eqn:compcond} and \Cref{lem:Wtilde} we ensure $G_{i,m}([0,T]\times W)\subseteq \tilde{W}$. Furthermore, $\Phi([0,T],W)\subseteq \tilde{W}$ by construction.
	Again, with \cref{eqn:compcond} the conditions of \Cref{lem:Wp} are fulfilled. Therefore, $\tilde{W}_P$ is well defined and compact.
	Also, all derivatives $(D^\beta f^{(k)})([0,T],\tilde{W})$. for $f=\Phi, G_{i,m}$ are always in $\tilde{W}_P$.\\
	Using the triangle inequality repeatedly, we have
	\begin{align*}
		&\normt{\Phi-G_{m,m}}=\norm{\tilde{\Phi}_{m.m}-\tilde{G}_{m,m}}_{C^r([0,T]\times W; \R^{n+1})}\\
		\leq&\norm{(\tilde{\Phi}_{1,m}-\tilde{g}_{m,m})\circ \tilde{G}_{m-1,m}}_{C^r([0,T]\times W; \R^{n+1})}\\
		&+\sum_{i=0}^{m-1}\norm{\tilde{\Phi}_{i,m}\circ \tilde{g}_{m-i,m} \circ \tilde{G}_{m-i-1,m}-
			\tilde{\Phi}_{i,m}\circ \tilde{\Phi}_{1,m}\circ \tilde{G}_{m-i-1,m}}_{C^r([0,T]\times W; \R^{n+1})}
	\end{align*}
	To estimate the right hand side, we observe each summand individually. We start with the easier part, i.e., the first summand. \Cref{lem:Wtilde} and \Cref{lem:Wp} ensure, that the conditions of \Cref{lem:lem2} are fulfilled, if we choose $f_1:=\tilde{\Phi}_{1,m}^{(j)},~f_2:=\tilde{g}_{m,m}^{(j)},~g:=\tilde{G}_{m-1,m}$ for $j=1,..,n$. Using \Cref{lem:lem2} leads to
	\begin{align*}
		&\normt{(\tilde{\Phi}_{1,m}-\tilde{g}_{{m}})\circ \tilde{G}_{m-1,m}}\\
		=&\sum_{\abs{\beta}\leq r}\max_{1\leq j\leq n}\sup_{(h,x)\in[0,T]\times W}
		\abs{D^\beta\left(\tilde{g}_{m,m}^{(j)}\circ\tilde{G}_{m-1,m}\right)(h,x)-D^\beta\left(\tilde{\Phi}_{1,m}^{(j)}\circ\tilde{G}_{m-1,m}\right)(h,x)}\\
		\leq& \sum_{\abs{\beta}\leq r}\max_{1\leq j\leq n} C_{\tilde{W},\tilde{W}_P,\beta}\norm{\tilde{g}_{m,m}^{(j)}-\tilde{\Phi}_{1,m}^{(j)}}_{C^r([0,T]\times\tilde{W})}.
	\end{align*}
	Now we take a closer look at the second summand
	\begin{align*}
		&\normt{\tilde{\Phi}_{i,m}\circ \tilde{g}_{m-i,m} \circ \tilde{G}_{m-i-1,m}-
			\tilde{\Phi}_{i,m}\circ \tilde{\Phi}_{1,m}\circ \tilde{G}_{m-i-1,m}}\\
		=&\sum_{\abs{\alpha}\leq r} \max_{1\leq k\leq n}\sup_{(h,x) \in [0,T]\times W}\left\vert D^\alpha\left(\tilde{\Phi}_{i,m}^{(k)}\circ \tilde{g}_{m-i,m} \circ \tilde{G}_{m-i-1,{m}}\right)(h,x)\right.\\
		&\qquad\qquad\qquad\qquad\qquad\qquad\left.- D^\alpha \left(\tilde{\Phi}^{(k)}_{i,m}\circ \tilde{\Phi}_{1,m} \circ \tilde{G}_{m-i-1,m}\right)(h,x)\right\vert.
	\end{align*}
	We want to apply \Cref{lem:lem1} with $f:=\tilde{\Phi}_{i,m}^{(k)},~ g_1:=\tilde{g}_{m-i,m}\circ \tilde{G}_{m-i-1,m},~g_2:=\tilde{\Phi}_{1,m}\circ \tilde{G}_{m-i-1,{m}}$. With \Cref{lem:Wtilde} and \Cref{lem:Wp} it is easy to see that almost all the assumptions of \Cref{lem:lem1} are fulfilled by our choices of $f,g_1,g_2, \tilde{W}$ and $\tilde{W}_P$. We only need to show $\tilde{\Phi}_{1,m}\circ \tilde{G}_{i,m} \in [0,T]\times \tilde{W}$. To do this, we use \cref{eqn:LPhi,eq:ivWtilde} and observe
	\begin{align*}
		&\norm{\tilde{\Phi}_{1,m}\circ \tilde{G}_{i,m}- \tilde{\Phi}_{i+1,m}}_{C([0,T]\times W;\R^{n+1})}\\
        &=\sup_{(h,x)\in [0,T]\times W}\norm{{\Phi}(h/m,{G}_{i,m}(h,x))- {\Phi}(h/m,\Phi_{i,m}(h,x))}_\infty\\
		&\leq \left(1+ \frac{C_\Phi}{m}\right)\norm{G_{i,m}-\Phi_{i,m}}_{C([0,T]\times W; \R^n)}\\
		&\leq \left(1+ \frac{C_\Phi}{m}\right)^{i+1} \frac{iR}{m\exp(C_\Phi)}\leq R.
	\end{align*}
	Therefore, $\tilde{\Phi}_{1,m}\circ \tilde{G}_{i,m} \in [0,T]\times \tilde{W}$. and by applying \Cref{lem:lem1}, we get
	\begin{align*}
		&\sup_{(h,x)\in [0,T]\times W}\abs{D^\alpha\left(\tilde{\Phi}_{i,m}^{(k)}\circ \tilde{g}_{m-i,m} \circ \tilde{G}_{m-i-1,{m}}\right)(h,x)-D^\alpha \left(\tilde{\Phi}^{(k)}_{i,m}\circ \tilde{\Phi}_{1,m} \circ \tilde{G}_{m-i-1,m}\right)(h,x)}\\
		\leq& C_{\tilde{W},\tilde{W}_P,\alpha,\tilde{\Phi}^{(k)}_{i,m}}
		\norm{\tilde{g}_{m-i,m}\circ \tilde{G}_{m-i-1,m}- \tilde{\Phi}_{1,m}\circ \tilde{G}_{m-i-1,m}}_{C^r([0,T]\times W; \R^{n+1})}.
	\end{align*}
	If we take a closer look at the proof of \Cref{lem:lem1}, we observe that $C_{\tilde{W},\tilde{W}_P,\alpha,\tilde{\Phi}^{(k)}_{i,m}}$ depends on $\tilde{\Phi}^{(k)}_{i,m}$ through its Lipschitz constant on $\tilde{W}$ and $\norm{D^\beta\Phi^{(k)}_{i,m}}_{C(\tilde{W})}$ for $\abs{\beta}\leq \abs{\alpha}$. Hence, $C_{\tilde{W},\tilde{W}_P,\alpha,\tilde{\Phi}^{(k)}_{i,m}}\leq C_{\tilde{W},\tilde{W}_P,\alpha,\tilde{\Phi}^{(k)}}$ follows because $\tilde{\Phi}^{(k)}_{i,m}(h,x)=\tilde{\Phi}^{(k)}(ih/m,x)$. Now we can apply \Cref{lem:lem2} with $f_1:=\tilde{g}^{(j)}_{m-i,m},~f_2=\tilde{\Phi}_{1,m}^{(j)},~g:=\tilde{G}_{m-i-1,m}$. as the condition on $\bar{W}_P$ is fulfilled by $\tilde{W}_P$ through \Cref{lem:Wp}.
	\begin{align*}
		&\norm{\tilde{g}_{m-i,m}\circ \tilde{G}_{m-i-1,m}- \tilde{\Phi}_{1,m}\circ \tilde{G}_{m-i-1,m}}_{C^r([0,T]\times W; \R^{n+1})}\\
		=&\sum_{\abs{\beta}\leq r} \max_{1\leq j\leq n}\sup_{(h,x) \in [0,T]\times W}\abs{D^\beta\left(\tilde{g}_{m-i,m}^{(j)}\circ\tilde{G}_{m-i-1,m}\right)(h,x)-D^\beta\left(\tilde{\Phi}_{1,m}^{(j)}\circ\tilde{G}_{m-i-1,m}\right)(h,x)}\\
		\leq&\sum_{\abs{\beta}\leq r} \max_{1\leq j\leq n}C_{\tilde{W},\tilde{W}_P,\beta}\norm{\tilde{g}_{m-i,m}^{(j)}-\tilde{\Phi}_{1,m}^{(j)}}_{C^r([0,T]\times\tilde{W})}.
	\end{align*}
	In total, we arrive at
	\begin{align*}
		&\normt{\Phi-G_m}=\norm{\tilde{\Phi}_{m,m}-\tilde{G}_{m,m}}_{C^r([0,T]\times W; \R^{n+1})}\\
		\leq&\norm{(\tilde{\Phi}_{1,m}-\tilde{g}_{m,m})\circ \tilde{G}_{m-1,m}}_{C^r([0,T]\times W; \R^{n+1})}\\
		&+\sum_{i=0}^{m-1}\norm{\tilde{\Phi}_{i,m}\circ \tilde{g}_{m-i,m} \circ \tilde{G}_{m-i-1,m}-
			\tilde{\Phi}_{i,m}\circ \tilde{\Phi}_{1,m}\circ \tilde{G}_{m-i-1,m}}_{C^r([0,T]\times W; \R^{n+1})}\\
		\leq&\sum_{\abs{\beta}\leq r}\max_{1\leq j\leq n} C_{\tilde{W},\tilde{W}_P,\beta}\norm{\tilde{g}_{m,m}^{(j)}-\tilde{\Phi}_{1,m}^{(j)}}_{C^r([0,T]\times\tilde{W})}\\
		&+\sum_{i=0}^{m-1}\sum_{\abs{\alpha}\leq r}\max_{1\leq k\leq n}C_{\tilde{W},\tilde{W}_P,\alpha,\tilde{\Phi}^{(k)}}\sum_{\abs{\beta}\leq r} \max_{1\leq j\leq n}C_{\tilde{W},\tilde{W}_P,\beta}\norm{\tilde{g}_{m-i,m}^{(j)}-\tilde{\Phi}_{1,m}^{(j)}}_{C^r([0,T]\times\tilde{W})}\\
		\leq&\bar{C}m\max_{i=1,..,m}\norm{g_{i,m}-\Phi_{1,m}}_{C^r([0,T]\times \tilde{W};\R^n)}
	\end{align*}
	with 
	\begin{align*}
		\bar{C}:=\left(1+\sum_{\abs{\alpha}\leq r}\max_{1\leq k\leq n}C_{\tilde{W},\tilde{W}_P,\alpha,\tilde{\Phi}^{(k)}}\right)\sum_{\abs{\beta}\leq r}C_{\tilde{W},\tilde{W}_P,\beta}.
	\end{align*}
	Finally, \cref{eq:comp} follows with \cref{eqn:compcond} and $\tilde{C}:=\bar{C}C$.
\end{proof}
\Cref{the:comp} holds for non-autonomous ODEs as well. This is shown in the following corollary.
\begin{cor}
	\label{the:compcor}
	Let $n, m, r\in \N,~ U \subseteq \R^n$ open, $W_t \subseteq \R$ compact, $f \in C^r(U \times W_t;\R^n)$ and $\Phi:\R^+_0 \times W_t\times U \to \R^n$ be the phase flow of the ODE
	\begin{align}
		\label{eq:nauto}
		\dot{x}(t)=f(x,t), \qquad x(t_0)=x_0,
	\end{align}
	i.e.\ $\Phi(t,t_0,x_0)=x(t)$. Let $T>0$ such that $\Phi(t,t_0,x)$ exists for all $t \in [0,T],~t_0\in W_t,~x \in U$ and $\Phi([0,T]\times W_t\times U) \subseteq U$. We define $\Phi_{1,i,m}(h,t,x):=\Phi(h/m,t+ih/m,x)$. Furthermore, let $g_i \in C^r([0,T]\times \R\times\R^n;\R^n)$, $g_{i,m}(h,t,x):=g_i(h/m,t,x)$ and 
	\begin{align*}
		G_{i,m}(h,t,x):=g_{i,m}\left(h,t+\frac{(i-1)h}{m},g_{i-1,m}\left(h,t+\frac{(i-2)h}{m},...,g_{1,m}(h,t,x)...\right)\right)
	\end{align*}
	for $i=1,...,m$. For a compact set $W \subseteq U$ we define $W_\Phi:=\Phi([0,T],W_t,W)$ and $\tilde{W}:=\overline{B_R(W_\Phi)}$. with $R < \dist(W_\Phi,\partial U)$. Also, we define $\tilde{W}_t:=\left\{t_0+t \strich t_0 \in W_t,~t \in [0,T]\right\}$.
	Let there be constant $C<\infty$ independent of $m$ such that
	\begin{align*}
		\norm{g_{i,m}-\Phi_{1,i,m}}_{C^r([0,T]\times \tilde{W}_t \times \tilde{W};\R^n)}\leq \frac{C}{m^2}
	\end{align*}
	for $i=1,...,m$. Then there is $m_0 \in \N$ such that for all $m \geq m_0$ the difference between the composition $G_{m,m}$ and the phase flow $\Phi$ decreases with $1/m$, i.e.,
	\begin{align*}
		\norm{\Phi-G_{m,m}}_{C^r([0,T]\times W_t \times W;\R^{n})}\leq \frac{\tilde{C}}{m},
	\end{align*}
	where $\tilde{C}$ does not depend on $m$.
\end{cor}
\begin{proof}
	We can rewrite \cref{eq:nauto} as an autonomous ODE by extending the phase space
	\begin{align}
		\label{eq:extode}
		\dot{y}=
		\begin{pmatrix}
			\dot{x}\\
			\dot{t}
		\end{pmatrix}=
		\begin{pmatrix}
			f(x,t)\\
			1
		\end{pmatrix}=\hat{f}(y).
	\end{align}
	This way we have $\Phi(t, t_0, x_0)=\hat{\Phi}(t, (x_0^T,t_0)^T)$. with $\hat{\Phi}$ being the phase flow of \cref{eq:extode}. Let 
	\begin{align*}
		\hat{g}_{i,m}\left(h, \begin{pmatrix}
			x\\ t
		\end{pmatrix}\right):=\begin{pmatrix}
			g_{i,m}(h,t,x)\\ t+ \frac{h}{m}
		\end{pmatrix},
	\end{align*}
	$\hat{\Phi}_{1,m}(h,x)=\hat{\Phi}(h/m,x),~\hat{G}_{i,m}:=\hat{g}_{i,m}(h,\hat{g}_{i-1,m}(h,...\hat{g}_{1,m}(h,m)...))$ and $\hat{W}:=\tilde{W}\times \tilde{W}_t$. Then we have
	\begin{align*}
		\norm{\hat{g}_{i,m}-\hat{\Phi}_{1,m}}_{C^r([0,T]\times\hat{W};\R^{n+1})}=\norm{g_{i,m}-\Phi_{1,i,m}}_{C^r([0,T]\times \tilde{W}_t \times \tilde{W};\R^n)}\leq \frac{C}{m^2}.
	\end{align*}
	Hence, the conditions of \Cref{the:comp} are fulfilled, and we get
	\begin{align*}
		\norm{\Phi-G_{m,m}}_{C^r([0,T]\times\tilde{W}_t\times\tilde{W};\R^n)}=\norm{\hat{\Phi}-\hat{G}_{m,m}}_{C^r([0,T]\times \hat{W};\R^{n+1})}\leq \frac{\tilde{C}}{m}
	\end{align*}
\end{proof}
Now the proof of \Cref{the:turaev2} can be done, following the proof in \cite{Tur03}, but using \Cref{the:compcor} instead of \cref{eqn:turaev_false}.

%% file: theory.tex
\section{Theory of time-adaptive SympNets}
\label{sec:app}
SympNets \cite{JinZZetal20} focus on Hamiltonian systems given by
\begin{align}
	\label{eqn:Hsys}
	\dot{x}(t)= J^{-1} \nabla H(x(t)), \qquad x(t_0)=x_0
\end{align}
with $x= (p^T, q^T)^T \in \R^{2d}$ and 
\begin{align*}
	J=\begin{pmatrix}
		0 & I_d\\
		-I_d & 0
	\end{pmatrix}\in \R^{2d \times 2d}.
\end{align*}
They are trying to approximate the flow of a Hamiltonian systems \eqref{eqn:Hsys} for a given fixed time-step $h$. Since the flow $\Phi$ of a Hamiltonian system is symplectic \cite{HaiLW06}, SympNets were designed to be symplectic by construction of their architecture. We will not recap the definition of SympNets here as their definitions are basically given by \Cref{def:sympnet,def:old_la_sympnet} if we fix the time step $h=1$.\\
In this work, we want to extend the SympNet architecture proposed in \cite{JinZZetal20} such that it is able to handle different time steps $h$ and call this extension time-adaptive SympNets. We also want to enforce $\psi(0,x)=x$ for any time-adaptive SympNet, since this is a fundamental property of the flow $\Phi$.
\subsection{Time-adaptive SympNets for autonomous Hamiltonian systems}
\label{subsec:td_SympNet}
There are two kinds of SympNets proposed in \cite{JinZZetal20}, LA-SympNets and G-SympNets. We will consider G-SympNets first, since we will use them in the proof of \Cref{the:g-sympnet}.
To define the time-adaptive G-SympNets we have to look at time-adaptive symplectic gradient modules first. It is recommended by \cite{JinZZetal20} to insert a time step $t$ into the symplectic modules to deal with irregular sampled data in the following way.

\begin{defi}[Time-adaptive G-SympNet, \cite{JinZZetal20}]\label{def:sympnet}
	For $n \in \N$ and a given activation function $\sigma$ we define upper and lower gradient modules by
	\begin{align*}
		\mathcal{G}_\text{up}\left(h,\begin{pmatrix}
			p\\q
		\end{pmatrix}\right):=\begin{bmatrix}
			I & h\hat{\sigma}_{K,a,b}\\
			0 & I
		\end{bmatrix}\begin{pmatrix}
			p\\ q
		\end{pmatrix}, \qquad
		\mathcal{G}_\text{low}\left(h,\begin{pmatrix}
			p\\q
		\end{pmatrix}\right):=\begin{bmatrix}
			I & 0\\
			h\hat{\sigma}_{K,a,b} & I
		\end{bmatrix}\begin{pmatrix}
			p\\ q
		\end{pmatrix},
	\end{align*}
	with $\hat{\sigma}_{K,a,b}(x):=K^T \diag (a)\sigma(K x+b)$, where $K \in \R^{n \times d},~a, b \in \R^n$. The set of gradient modules $\mathcal{M}_G$ is given by
	\begin{align*}
		\mathcal{M}_{\text{G}}:=\left\{u \vert u \text{ is a gradient module}\right\}.
	\end{align*}
	Now we can define the set of \textbf{time-adaptive G-SympNets (TG-SympNets)} as
	\begin{align*}
		\Psi_{\text{G}}:=\left\{\psi(h,x) =(u_k(h,\cdot) \circ ... \circ u_1(h,\cdot))(x)\vert u_k,...,u_1 \in \mathcal{M}_G, k \in \N\right\}
	\end{align*}
\end{defi}
\begin{figure}
	\centering
	\includegraphics[width=\textwidth]{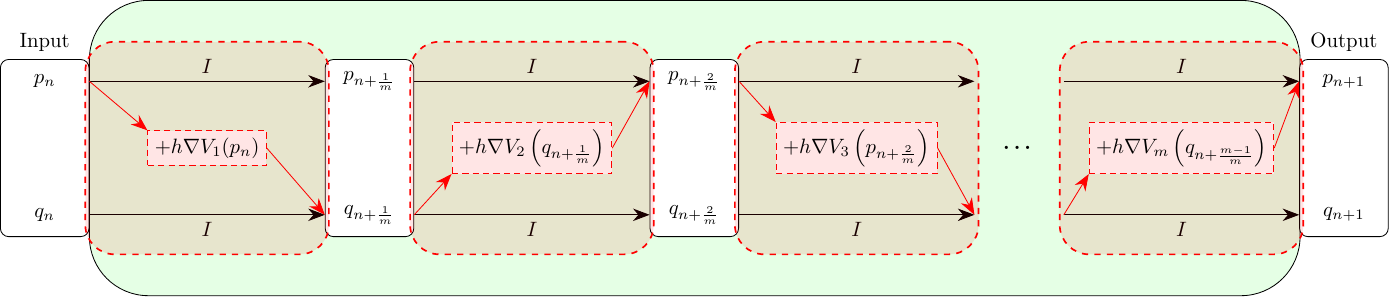}
	\caption{Architecture of TSympNets for autonomous separable Hamiltonian systems}
	\label{fig:SympNet}
\end{figure}
Time-adaptive LA-SympNets can be defined similarly.
\begin{defi}[Original TLA-SympNet \cite{JinZZetal20}]
    \label{def:old_la_sympnet}
	First we define linear modules by
	\begin{align*}
		\tilde{\mathcal{L}}_n^{\text{up}}\left(h,\begin{pmatrix}
			p\\ q
		\end{pmatrix}\right)&=
		\begin{pmatrix}
			I & 0/hS_n\\ hS_n/0 & I
		\end{pmatrix}...\begin{pmatrix}
			I & 0\\ hS_2 & I
		\end{pmatrix}
		\begin{pmatrix}
			I & hS_1\\ 0 &I
		\end{pmatrix}
		\begin{pmatrix}
			p\\ q
		\end{pmatrix}+ hb,\\
		\tilde{\mathcal{L}}_n^{\text{low}}\left(h,\begin{pmatrix}
			p\\ q
		\end{pmatrix}\right)&=
		\begin{pmatrix}
			I & 0/hS_n\\ hS_n/0 & I
		\end{pmatrix}...\begin{pmatrix}
			I & hS_2\\ 0 & I
		\end{pmatrix}
		\begin{pmatrix}
			I & 0\\ hS_1 & I
		\end{pmatrix}
		\begin{pmatrix}
			p\\ q
		\end{pmatrix}+ hb
	\end{align*}
	with $S_1,...,S_n \in \R^{d \times d}$ symmetric and $b\in \R^{2d}$. We write
	\begin{align*}
		\tilde{\mathcal{M}}_L:=\left\{v \vert v \text{ is a linear module}\right\}
	\end{align*} for the set of linear modules. For an activation function $\sigma$ the corresponding activation modules are given by
	\begin{align*}
		\tilde{\mathcal{N}}_\text{up}\left(h,\begin{pmatrix}
			p\\ q
		\end{pmatrix}\right)=\begin{bmatrix}
			I & h\tilde{\sigma}_{a}\\ 0 & I
		\end{bmatrix}
		\begin{pmatrix}
			p\\ q
		\end{pmatrix}, \qquad
		\tilde{\mathcal{N}}_\text{low}\left(h,\begin{pmatrix}
			p\\ q
		\end{pmatrix}\right)=\begin{bmatrix}
			I & 0\\ h\tilde{\sigma}_{a} & I
		\end{bmatrix}
		\begin{pmatrix}
			p\\ q
		\end{pmatrix}
	\end{align*}
	with $\tilde{\sigma}_{a}(x):=\diag(a)\sigma(x)$, where $a,b \in \R^d$. We denote the set of activation modules by
	\begin{align*}
		\tilde{\mathcal{M}}_{\text{A}}=\left\{w \vert w \text{ is an activation module}\right\}.
	\end{align*}
	Now we can define the set of \textbf{original time-adaptive LA-SympNets (OTLA-SympNets)} as
	\begin{align*}
		\tilde{\Psi}_{\text{LA}}:=\bigg\{&\psi(h,x)=(v_{k+1}(h,\cdot) \circ w_k(h,\cdot) \circ v_k(h,\cdot) \circ ... \circ w_1(h,\cdot) \circ v_1(h,\cdot))(x)\\
		&\left. \strich v_1,...,v_{k+1}\in \tilde{\mathcal{M}}_{\text{L}}, w_1, ..., w_k \in \tilde{\mathcal{M}}_{\text{A}}, k \in \N\right\}
	\end{align*}
\end{defi}
Both TG-SympNets and TLA-SympNets follow the concept of using layers that update $p$ and $q$ separately. This is visualized in \Cref{fig:SympNet}.\\
Because of the limitations of the original TLA-SympNets discussed in \Cref{subsec:lim_og_LA}, we propose to use a different architecture, that has some theoretical benefits. 
\begin{defi}[New TLA-SympNets]
	\label{def:new_la_sympnet}
	We define new linear modules by
	\begin{align*}
		\mathcal{L}_n\pq&:= \begin{pmatrix}
			I  & 0/S_n\\
			S_n/0 & I
		\end{pmatrix}\hdots \begin{pmatrix}
			I & 0\\ S_2 & I
		\end{pmatrix}\begin{pmatrix}
			I & S_1\\ 0 & I
		\end{pmatrix}\pq
	\end{align*}
	with $S_1, ..., S_n\in \R^{d\times d}$ symmetric and write 
	\begin{align*}
		\mathcal{M}_{\text{L}}:=\left\{v \vert v \text{ is a linear module as defined above}\right\}
	\end{align*} for the set of linear modules. For an activation function $\sigma$ the corresponding activation modules are given by
	\begin{align*}
		\mathcal{N}^{\text{up}}\left(h,\pq\right)&= \begin{bmatrix}
			I & h\tilde{\sigma}_{a, b}\\ 0 & I
		\end{bmatrix}\pq, \qquad
		\mathcal{N}^{\text{low}}\pq= \begin{bmatrix}
			I & 0 \\ h\tilde{\sigma}_{a, b} & I
		\end{bmatrix}\pq,
	\end{align*}
	with $\tilde{\sigma}_{a,b}(x):=\diag(a)\sigma(x+b)$ and $a, b \in \R^d$. 
	We denote the set of activation modules by
	\begin{align*}
		{\mathcal{M}}_{\text{A}}=\left\{w \vert w \text{ is an activation module as defined above}\right\}.
	\end{align*}
	Now we define the set of \textbf{TLA-SympNets} as
	\begin{align*}
		{\Psi}_{\text{LA}}:=\{&\psi(h,x) = (v_k^{-1}\circ w_k(h,\cdot) \circ v_k \circ \hdots \circ v_1^{-1}\circ w_1(h,\cdot) \circ v_1)(x)\\
		&\left.\strich v_1,...,v_{k} \in \mathcal{M}_{\text{L}},~ w_1,...,w_k \in \mathcal{M}_{\text{A}}\right\}.
	\end{align*}
\end{defi}
The reasoning behind the made changes is as follows.
First, we removed the time input from the linear layers. We do this, because the TLA-SympNets proposed in \cite{JinZZetal20} are limited to approximating Hamiltonian systems with Hamiltonians structured in a very specific form in $\normtd{\cdot}$. This limitation will be shown in \Cref{subsec:lim_og_LA}. 
Second, we moved the bias $b$ from the linear modules to the activation modules, which reduces the parameter space by $d$ parameters per layer and allows us to invert the TLA-SympNets easier. Lastly, we added the inverse of the linear layers. This change had to be made to ensure $\psi(0,x)=x$ for all TLA-SympNets $\psi \in \Psi_{\text{LA}}$. The inversion of the linear layers is as cheap as the forward evaluation of said layers since
\begin{align*}
	\mathcal{L}_n^{-1}\pq= \begin{pmatrix}
		I  & -S_1\\
		0 & I
	\end{pmatrix} \begin{pmatrix}
		I & 0\\ -S_2 & I
	\end{pmatrix}\hdots\begin{pmatrix}
		I & 0/-S_k\\ -S_k/0 & I
	\end{pmatrix}\pq.
\end{align*}
Obviously, these changes do not impact the symplecticity of the linear or activation layers. Also, this would not impact the approximation capabilities of the original (not time-adaptive, i.e., $h=1$) Sympnets, as we only need to replace
\begin{align*}
	v_i\begin{pmatrix}
		p \\ q
	\end{pmatrix}=
	\begin{pmatrix}
		K_i^{-1} & 0 \\ 0 & K_i^{T}
	\end{pmatrix}
	\begin{bmatrix}
		I & 0 \\ \tilde{\sigma}_{a_i} & I
	\end{bmatrix}
	\left(
	\begin{pmatrix}
		K_i & 0 \\ 0 & K_i^{-T}
	\end{pmatrix}\begin{pmatrix}
		p \\q
	\end{pmatrix}+
	\begin{pmatrix}
		b_i \\0
	\end{pmatrix}
	\right)-
	\begin{pmatrix}
		K_i^{-1} b_i \\0
	\end{pmatrix}
\end{align*}
in the proof of \cite[Theorem 4]{JinZZetal20} by
\begin{align*}
	\tilde{v}_i\begin{pmatrix}
		p \\ q
	\end{pmatrix}=
	\begin{pmatrix}
		K_i^{-1} & 0 \\ 0 & K_i^{T}
	\end{pmatrix}
	\begin{bmatrix}
		I & 0 \\ \tilde{\sigma}_{a_i,b_i} & I
	\end{bmatrix}
	\left(
	\begin{pmatrix}
		K_i & 0 \\ 0 & K_i^{-T}
	\end{pmatrix}\begin{pmatrix}
		p \\q
	\end{pmatrix}
	\right)
\end{align*}
and one can easily verify that $v_i=\tilde{v}_i$. In \Cref{subsec:uat}, we will prove that this new architecture has a universal approximation theorem even for varying $h$.

\subsection{Time-adaptive SympNets for non-autonomous Hamiltonian systems}
\label{subsec:na_sympnets}
Now we want to focus on the non-autonmous case $H(p,q,t)$. Even though it is easily possible to extend a non-autonomous Hamiltonian system into an autonomous one \cite{Lan49}, we will not make use of that phase space extension in this work, as the extended system might not be separable even if the original one is. This observation can be made the following way. For the conjugate coordinate to $t$, let us call it $e$, we get $\dot{e}=-\partial_t H(p,q,t)$ as the equation of motion. So with $\tilde{p}=(p^T,e)^T,~ \tilde{q}=(q^T, t),~\tilde{H}(\tilde{p},\tilde{q})=H(p,q,t) + e$ we obtain the same equations of motion as in the original non-autonomous system
\begin{align*}
	\begin{pmatrix}
		\dot{p}\\ \dot{e}
	\end{pmatrix}=\dot{\tilde{p}}=-\frac{\partial\tilde{H}}{\partial\tilde{q}}=
	\begin{pmatrix}
		-\frac{\partial \tilde{H}}{\partial q}\\ -\frac{\partial \tilde{H}}{\partial t}
	\end{pmatrix}=
	\begin{pmatrix}
		-\frac{\partial H}{\partial q}\\ -\frac{\partial H}{\partial t}
	\end{pmatrix},\qquad
	\begin{pmatrix}
		\dot{q} \\ \dot{t}
	\end{pmatrix}=\dot{\tilde{q}}=\frac{\partial\tilde{H}}{\partial\tilde{p}}=
	\begin{pmatrix}
		\frac{\partial \tilde{H}}{\partial p}\\  \frac{\partial \tilde{H}}{\partial e}
	\end{pmatrix}=
	\begin{pmatrix}
		\frac{\partial H}{\partial p}\\ 1
	\end{pmatrix}.
\end{align*}
Now assume that we have a separable Hamiltonian $H(p,q,t)=K(p,t)+V(q,t)$. Then $\tilde{H}$ is not separable, since
\begin{align*}
	\frac{\partial \tilde{H}}{\partial \tilde{q}}=\begin{pmatrix}
		\partial_q V(q, t)\\
		\partial_t K(p, t) + \partial_t V(q, t)
	\end{pmatrix}
\end{align*}
also depends on $p$. Hence, we don't use this phase space extension to deal with non-autonomous systems, as we show in \Cref{subsec:lim} that separability is an important condition, if we want to use adaptive time-steps. Also, one would need measurements of $e$, i.e., the original Hamiltonian. \\
Note that the extension would be applicable, if only $V$ was depending on $t$ but not $K$, i.e., $H(p,q,t)=K(p)+V(q, t)$. So this might be a possibility to consider in future research.\\
Since $K$ and $V$ depend on time, we will allow the same for the potentials used in the SympNets. So the updates made by the non-autonomous SympNets try to approximate
\begin{align*}
	\begin{bmatrix}
		I & h\nabla V(\cdot,t)\\
		0 & I
	\end{bmatrix}\pq, \qquad \begin{bmatrix}
		I & 0 \\
		h\nabla K(\cdot,t) & I
	\end{bmatrix}\pq,
\end{align*}
and update $t\mapsto t+ h/m$ for a given step size $h$, with $m$ being the number of layers. This concept allows us to define non-autonomous TG-Sympets and is visualized in \Cref{fig:na_SympNet}.
\begin{figure}
	\centering
	\includegraphics[width=\linewidth]{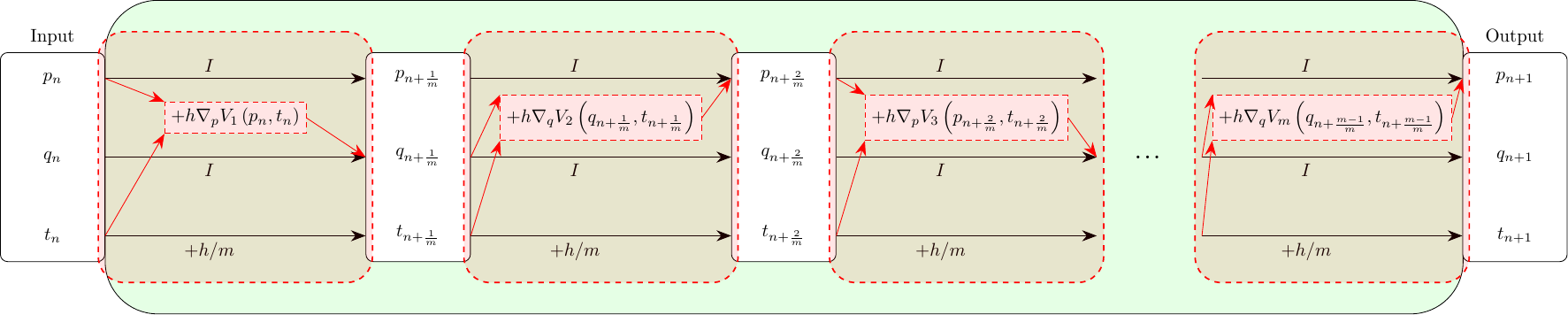}
	\caption{TSympNets for non-autonomous separable Hamiltonian systems}
	\label{fig:na_SympNet}
\end{figure}

\begin{defi}[Non-autonomous TG-SympNets]
	For $n \in \N$ and a given activation function $\sigma$, we define 
	\begin{align*}
		\mathcal{G}_{\text{up}}\left(h, t, \pq\right):=\begin{bmatrix}
			I & h \hat{\sigma}_{K, a, b, c}(\cdot, t)\\
			0 & I
		\end{bmatrix}\pq,
		\qquad
		\mathcal{G}_{\text{low}}\left(h, t, \pq\right):=\begin{bmatrix}
			I & 0\\
			h \hat{\sigma}_{K, a, b, c}(\cdot, t) & I
		\end{bmatrix}\pq,
	\end{align*}
	with $\hat{\sigma}_{K,a,b,c}(x, t):=K^T \diag(a) \sigma(K x + c t +b)$, where $K \in \R^{n \times d},~a,b,c\in \R^n$. Now the upper and lower gradient modules are defined by
	\begin{align*}
		u_{\text{up}}^{(m)}\left(h, t, \pq \right):=\begin{pmatrix}
			t + \frac{h}{m}\\
			\mathcal{G}_{\text{up}}\left(h, t, \pq\right)
		\end{pmatrix}, \qquad u_{\text{low}}^{(m)}\left(h, t, \pq \right):=\begin{pmatrix}
			t + \frac{h}{m}\\
			\mathcal{G}_{\text{low}}\left(h, t, \pq\right)\\
		\end{pmatrix},
	\end{align*}
	with $m \in \N$. We denote the set of gradient modules as
	\begin{align*}
		\mathcal{M}_{\text{NAG}}^{(m)}:=\left\{u^{(m)}\strich u \text{ is a gradient module as defined above}\right\}.
	\end{align*}
	Now we define the set of \textbf{non-autonomous TG-Sympnets (NATG-SympNets)} as
	\begin{align*}
		\Psi_{\text{NAG}}:=\left\{\psi(h, t, x) = (u_m(h, \cdot, \cdot) \circ ... \circ u_1(h, \cdot, \cdot))(t,x) \strich u_m, ..., u_1 \in \mathcal{M}_{\text{NAG}}^{(m)}\right\}.
	\end{align*}
\end{defi}
With a similar extension we can create non-autonomous LA SympNets from the ones defined in \Cref{def:new_la_sympnet}
\begin{defi}[Non-autonomous TLA-SympNets]
	The linear modules stay the same as in \Cref{def:new_la_sympnet}
	\begin{align*}
		\mathcal{L}_n\pq&:= \begin{pmatrix}
			I  & 0/S_n\\
			S_n/0 & I
		\end{pmatrix}\hdots \begin{pmatrix}
			I & 0\\ S_2 & I
		\end{pmatrix}\begin{pmatrix}
			I & S_1\\ 0 & I
		\end{pmatrix}\pq
	\end{align*}
	with $S_1, ..., S_n\in \R^{d\times d}$ symmetric and we write 
	\begin{align*}
		\mathcal{M}_L:=\left\{v \vert v \text{ is a linear module}\right\}
	\end{align*} for the set of linear modules. For an activation function $\sigma$ the corresponding activation modules are given by
	\begin{align*}
		\mathcal{N}^{\text{up}}\left(h,\pq\right)&= \begin{bmatrix}
			I & h\tilde{\sigma}_{a, b, c}\\ 0 & I
		\end{bmatrix}\pq, \qquad
		\mathcal{N}^{\text{low}}\pq= \begin{bmatrix}
			I & 0 \\ h\tilde{\sigma}_{a, b, c} & I
		\end{bmatrix}\pq,
	\end{align*}
	with $\tilde{\sigma}_{a,b, c}(x):=\diag(a)\sigma(x+t c +b)$ and $a, b, c \in \R^d$. 
	We denote the set of activation modules by
	\begin{align*}
		{\mathcal{M}}_{\text{NAA}}=\left\{w \vert w \text{ is an activation module as defined above}\right\}.
	\end{align*}
	Now we define the set of \textbf{non-auonomous TLA-SympNets (NATLA-SympNets)} as
	\begin{align*}
		{\Psi}_{\text{NATLA}}:=&\left\{\psi(h, t, x) = (v_k^{-1}\circ w_k(h, \cdot,\cdot) \circ v_k \circ \hdots \circ v_1^{-1}\circ w_1(h, \cdot, \cdot) \circ v_1)(t,x)\right.\\
		&\left.\strich v_1,...,v_{k} \in \mathcal{M}_{\text{L}},~ w_1,...,w_k \in \mathcal{M}_{\text{NAA}}\right\}.
	\end{align*}
\end{defi}

\subsection{Universal Approximation theorems for time-adaptive SympNets}
\label{subsec:uat}
In this section we present universal approximation theorems for the different kinds of SympNets introduced in \Cref{subsec:td_SympNet,subsec:na_sympnets}.
To formulate our approximation theorems, we need the following concepts.
\begin{defi}[$r$-uniformly dense on compacta \cite{JinZZetal20}]
	Let $n_1, n_2 \in \N,~r \in \N_0,~T>0$ be given, $U \subseteq \R^{n_1}$ is an open set, $S_1 \subseteq C^r([0,T]\times U;\R^{n_2})$, then we say $S_2$ is \textbf{$r$-uniformly dense on compacta} in $S_1$ if $S_2 \subseteq S_1$ and for any $f \in S_1$, compact $W\subseteq U$ and $\epsilon>0$, there exists a $g \in S_2$ such that $\normt{f-g}<\epsilon$.
\end{defi}
\begin{defi}[$r$-finite \cite{JinZZetal20}]
	Let $r \in \N_0$. A function $\sigma \in C^r(\R)$ is called \textbf{$r$-finite} if
	\begin{align*}
		0< \int_\R \abs{D^r(\sigma)(x)} dx < \infty.
	\end{align*}
\end{defi}
Now we are able to present our approximation theorems. Note that all of the following theorems only guarantee that the proposed SympNets can approximate flows of arbitrary separable Hamiltonian systems. We will discuss in \Cref{subsec:lim} why we can not extend our results to non-separable Hamiltonian systems.
\begin{theorem}[Approximation theorem for TG-SympNets]
	\label{the:g-sympnet}
	For any $r \in \mathbb{N}$ the set of TG-SympNets $\Psi_{\text{TG}}$ is $r$-uniformly dense on compacta in
	\begin{align*}
		\left\{\Phi: \R_0^+ \times \R^{2d} \to \R^{2d} \mid \Phi \text{ is phase flow of \cref{eqn:Hsys}}\right\},
	\end{align*}
	where $H$ is given by
	\begin{align*}
		H(p,q)= K(p) + V(q)
	\end{align*}
	for some $K,V \in C^{r+1}(\R^{2d})$, if the activation function $\sigma$ is $r$-finite.
\end{theorem}
\begin{proof}
	Let $W\subseteq U$ be a compact set. 
	Since $C^{\infty}$ is dense in $C^{r+1}$ in the $\norm{\cdot}_{C^r(W;\R^{2d})}$-topology, we assume $K,V \in C^\infty$ for the rest of this proof. Hence, $H$ is also arbitrarily often differentiable. Therefore, this is also true for the flow of \cref{eqn:Hsys} $\Phi$ \cite[Appendix B, Lemma B.4]{DuiK12}.
	Furthermore, we will use the same definitions for $G_{i,m}$ and $\Phi_{i,m}$ as in \Cref{sec:corr}. We set
	\begin{align}\label{eq:ghm}
		g_{{m}}\left(h,\begin{pmatrix}
			p\\q
		\end{pmatrix}\right):&=
		\begin{bmatrix}
			I & 0\\
			\frac{h}{m} \nabla K & I
		\end{bmatrix}
		\begin{bmatrix}
			I & -\frac{h}{m} \nabla V\\
			0 & I
		\end{bmatrix}
		\begin{pmatrix}
			p\\q
		\end{pmatrix} \notag \\
		&=\begin{bmatrix}
			I & 0\\
			\frac{h}{m} \nabla K & I
		\end{bmatrix}\begin{pmatrix}
			p - \frac{h}{m}\nabla V(q)\\
			q
		\end{pmatrix}\notag\\
		&=
		\begin{pmatrix}
			p - \frac{h}{m}\nabla V(q)\\
			q + \frac{h}{m} \nabla K\left(p-\frac{h}{m}\nabla V(q)\right)
		\end{pmatrix}. 
	\end{align}
	Now we want to show that this choice of $g_m$ fulfills \cref{eqn:compcond} to apply \Cref{the:comp}. To do this, we look at different cases for $D^\alpha$. The first case is $D^\alpha =D^\alpha_x$ only including spatial derivatives and no temporal derivatives.
	We observe with Taylor's theorem
	\begin{align*}
		\Phi\left(\frac{h}{m},\begin{pmatrix}
			p\\ q
		\end{pmatrix}\right)&=\begin{pmatrix}
			p - \frac{h}{m}\nabla V(q) + \frac{h^2}{2 m^2}\partial_t^2 \Phi^{(p)}\left(\theta,\begin{pmatrix}
				p\\ q
			\end{pmatrix}\right)\\
			q + \frac{h}{m}\nabla K(p) + \frac{h^2}{2m^2}\partial_t^2\Phi^{(q)}\left(\theta, \begin{pmatrix}
				p\\ q
			\end{pmatrix}\right)
		\end{pmatrix}\\
		g_{{m}}\left(h,\begin{pmatrix}
			p\\q
		\end{pmatrix}\right)&=\begin{pmatrix}
			p -\frac{h}{m}\nabla V(q)\\
			q + \frac{h}{m}\nabla K(p) - \frac{h^2}{m^2}\nabla^2 K(p-\vartheta\nabla V(q))\nabla V(q)
		\end{pmatrix}
	\end{align*}
	with $\theta,\vartheta \in (0,h/m)$. So we have
	\begin{align*}
		&\sup_{(h,(p^T,q^T)^T\in [0,T]\times \tilde{W}}\max_{1\leq k \leq 2d}\abs{D^\alpha_x \left(g_m\left(h,\pq\right)-\Phi\left(\frac{h}{m},\pq\right)\right)}\\
		=&\sup_{(h,(p^T,q^T)^T\in [0,T]\times \tilde{W}}\max_{1\leq k \leq 2d}\abs{D^\alpha_x\begin{pmatrix}
				\frac{h^2}{2 m^2}\partial_t^2\Phi^{(p)}\left(\theta, \begin{pmatrix}
					p\\ q
				\end{pmatrix}\right)\\
				-\frac{h^2}{m^2}\nabla^2 K(p-\vartheta\nabla V(q))\nabla V(q) + \frac{h^2}{2m^2}\partial_t^2\Phi^{(q)}\left(\theta,\begin{pmatrix}
					p \\ q
				\end{pmatrix}\right) 
			\end{pmatrix}^{(k)}}\\
		\leq& \frac{1}{m^2}(C_{g}^{(0)}+C_{\Phi}^{(0)}),
	\end{align*}
	with
	\begin{align*}
		C_{g}^{(0)}&:=T^2\max_{\abs{\alpha}\leq r} \sup_{(h,(p^T,q^T)^T)\in [0,T]\times \tilde{W}} \max_{1\leq k \leq d} \abs{D^\alpha_x (\nabla^2K(p-h\nabla V(q))\nabla V(q))^{(k)}},\\
		C_\Phi^{(0)}&:=\frac{T^2}{2}\max_{\abs{\alpha}\leq r} \sup_{(h,x)\in [0,T]^2\times \tilde{W}}\max_{1\leq k \leq 2d}\abs{D^\alpha_x \partial_t^2 \Phi^{(k)}(h,x)}.
	\end{align*}
	The second case is $D^\alpha=D^\beta_x \partial_t$ including exactly one temporal derivative and $\abs{\beta}=\abs{\alpha}-1$ spatial derivatives. Using Taylor's theorem again we have
	\begin{align*}
		\frac{d}{dh} \Phi\left(\frac{h}{m},\pq\right)&=\begin{pmatrix}
			-\frac{1}{m} \nabla V(q) - \frac{h}{m^2} \nabla^2 V(q+ \theta \nabla K(p))\nabla K(p)\\
			\frac{1}{m} \nabla K(p) + \frac{h}{m^2}\nabla^2 K(p-\theta \nabla V(q))\nabla V(q)
		\end{pmatrix},\\
		\frac{d}{dh}g_m\left(h,\pq\right)&=\begin{pmatrix}
			-\frac{1}{m}\nabla V(q)\\
			\frac{1}{m}\nabla K(p) - \frac{h}{m^2}\nabla^2K(p-\vartheta\nabla V(q))\nabla V(q) -\frac{h}{m^2}\nabla^2K(p-\frac{h}{m}\nabla V(q))\nabla V(q)
		\end{pmatrix}.
	\end{align*}
	Therefore, it follows
	\begin{align*}
		&\sup_{(h,(p^T,q^T)^T\in [0,T]\times \tilde{W}}\max_{1\leq k \leq 2d}\abs{D^\alpha \left(g_m\left(h,\pq\right)-\Phi\left(\frac{h}{m},\pq\right)\right)}\\
		=&\sup_{(h,(p^T,q^T)^T\in [0,T]\times \tilde{W}}\max_{1\leq k \leq 2d}\abs{D^\beta \frac{d}{dh}\left(g_m\left(h,\pq\right)-\Phi\left(\frac{h}{m},\pq\right)\right)}\\
		\leq& \frac{1}{m^2}(C_{g}^{(1)}+C_\Phi^{(1)}),
	\end{align*}
	with
	\begin{align*}
		C_{g}^{(1)}&:=2T\max_{\abs{\beta}\leq r-1}\sup_{(h,(p^T,q^T))^T\in [0,T]\times \tilde{W}}\max_{1\leq k \leq d}\abs{D^\beta_x \left(\nabla^2 K(p-\vartheta\nabla V(q))\nabla V(q)\right)^{(k)}},\\
		C_\Phi^{(1)}&:=T\max_{\abs{\beta}\leq r-1}\sup_{(h,(p^T,q^T)^T\in [0,T]\times \tilde{W}}\max_{1\leq k \leq 2d}\abs{D^\beta_x\begin{pmatrix}
				\nabla^2V(q+ h \nabla K(p))\nabla K(p)\\
				\nabla^2 K(p-h \nabla V(q))\nabla V(q)
			\end{pmatrix}^{(k)}}.
	\end{align*}
	Lastly, we consider the case $D^\alpha=D^\beta_x\partial_t^n$ for $2\leq n\leq r$, i.e., $D^\alpha$ including at least two temporal derivatives. We define $g(h,x):=g_m(mh,x)$ and observe that $g$ is independent of $m$. Therefore,
	\begin{align}
        \begin{aligned}
        \label{eqn:higher_derivs}
		\sup_{(h,x)\in [0,T]\times \tilde{W}}\max_{1\leq k \leq 2d}\abs{D^\alpha g^{(k)}_m(h,x)}&=\sup_{(h,x)\in [0,T]\times \tilde{W}}\max_{1\leq k \leq 2d}\abs{D^\beta \frac{d^n}{dh^n} g^{(k)}(h/m,x)}\\
		&\leq \frac{C^{(n)}_g}{m^n},\\
		\sup_{(h,x)\in [0,T]\times \tilde{W}}\max_{1\leq k \leq 2d}\abs{D^\alpha \Phi^{(k)}(h/m,x)}&=
		\sup_{(h,x)\in [0,T]\times \tilde{W}}\max_{1\leq k \leq 2d}\abs{D^\beta \frac{d^n}{dh^n}\Phi^{(k)}(h/m,x)}\\
		&\leq \frac{C^{(n)}}{m^n},
        \end{aligned}
	\end{align}
	and together
	\begin{align*}
		\sup_{(h,x)\in [0,T]\times \tilde{W}}\max_{1\leq k \leq 2d}\abs{D^\alpha g^{(k)}_m(h,x)-\Phi^{(k)}(h/m,x)}\leq \frac{1}{m^n}(C_g^{(n)}+C_\Phi^{(n)}),
	\end{align*}
	with
	\begin{align*}
		C_g^{(n)}&:=\max_{\abs{\beta}\leq r-n}\sup_{(h,x)\in [0,T]\times \tilde{W}}\max_{1\leq k \leq 2d}\abs{D^\beta \partial_t^n g^{(k)}(h,x)},\\
		C_\Phi^{(n)}&:=\max_{\abs{\beta}\leq r-n}\sup_{(h,x)\in [0,T]\times \tilde{W}}\max_{1\leq k \leq 2d}\abs{D^\beta \partial_t^n \Phi^{(k)}(h,x)}.
	\end{align*}
	Combining all the cases, we get
	\begin{align*}
		\normtd{g_m-\Phi_{1,m}}&=\sum_{\abs{\alpha}\leq r}\sup_{(h,x)\in [0,T]\times \tilde{W}}\max_{1\leq k \leq 2d}\abs{D^\alpha(g_m(h,x)-\Phi_{1,m})^{(k)}}\\
		&\leq \binom{2d+r}{2d}\max_{n=1,...,r}(C_g^{(n)}+C_\Phi^{(n)})\frac{1}{m^2}\leq \frac{C}{m^2},
	\end{align*}
	with $C:=\binom{2d+r}{2d}\max_{n=1,...,r}(C_g^{(n)}+C_\Phi^{(n)})$. Hence, \cref{eqn:compcond} holds and we can apply \Cref{the:comp}, to get
	\begin{align*}
		\normtd{\Phi-G_m}\leq\frac{\tilde{C}}{m}\leq \epsilon
	\end{align*}
	for any $\epsilon>0$, if $m\geq \tilde{C}/\epsilon$.\\
	To show that there is a TG-SympNet which approximates $G_m$, one simply follows the proof of \cite[Theorem 5]{JinZZetal20} with
	\begin{align*}
		\mathcal{F}_G=\left\{f(t,x)=\begin{bmatrix}
			I & t \nabla V\\ 0 & I
		\end{bmatrix}(x)\strich V \in  C^{r+1}(\R^d)\right\}\cup
		\left\{f(t,x)=\begin{bmatrix}
			I & 0\\ t \nabla K & I
		\end{bmatrix}(x)\strich K \in  C^{r+1}(\R^d)\right\}
	\end{align*}
\end{proof}

We can obtain the same result for TLA-SympNets.
\begin{theorem}[Approximation Theorem for TLA-SympNets]
	\label{the:la-sympnet}
	For any $r \in \mathbb{N}$ and open $U \in \R^{2d}$, the set of TLA-SympNets $\Psi_{\text{TLA}}$ is $r$-uniformly dense on compacta in
	\begin{align*}
		\left\{\Phi: \R_0^+ \times U \to \R^{2d} \mid \Phi \text{ is phase flow of \cref{eqn:Hsys}}\right\},
	\end{align*}
	where $H$ is given by
	\begin{align*}
		H(p,q)= K(p) + V(q)
	\end{align*}
	for some $K,V \in C^{r+1}(U)$, if the activation function $\sigma$ is $r$-finite.
\end{theorem}
\begin{proof}
	Similar to the proof of \cite[Theorem 4]{JinZZetal20}, we can represent any TG-SympNet as a TLA-SympNet. Therefore, \Cref{the:la-sympnet} follows directly from \Cref{the:g-sympnet}. Another way of prooving \Cref{the:la-sympnet} would be to perform the proof of \Cref{the:naLA-sympnet} and discarding all terms related to $t$.
\end{proof}
Unfortunately it is not possible to show that the original TLA-SympNets from \cite{JinZZetal20} approximate the flow of separable Hamiltonian systems with arbitrary accuracy. This limitation is discussed in \Cref{subsec:lim_og_LA}.\\ 
Thanks to \Cref{the:compcor} we can obtain universal approximations for the non-autonomous TSympNets as well.
\begin{theorem}[Approximation theorem for NATG-SympNets]
	\label{the:naG-sympnet}
	For any $r \in \mathbb{N}$ and open $U \in \R^{2d}$, the set of NATG-SympNets $\Psi_{\text{NATG}}$ is $r$-uniformly dense on compacta in
	\begin{align*}
		\left\{\Phi: \R_0^+ \times \R_0^+\times U \to \R^{2d} \mid \Phi \text{ is phase flow of \cref{eqn:Hsys}}\right\},
	\end{align*}
	where $H$ is given by
	\begin{align*}
		H(p,q,t)= K(p,t) + V(q,t)
	\end{align*}
	for some $K,V \in C^{r+1}(U\times \R_0^+)$, if the activation function $\sigma$ is $r$-finite.
\end{theorem}
\begin{proof}
	Analogous to the proof of \Cref{the:g-sympnet}. Since the system is non-autonomous, we have to use \Cref{the:compcor} instead of \Cref{the:comp}.
\end{proof}
\begin{theorem}[Approximation theorem for non-autonomous LA-SympNets]
	\label{the:naLA-sympnet}
	For any $r \in \mathbb{N}$ and open $U \in \R^{2d}$, the set of NATLA-SympNets $\Psi_{NATLA}$ is $r$-uniformly dense on compacta in
	\begin{align*}
		\left\{\Phi: \R_0^+ \times \R_0^+\times U \to \R^{2d} \mid \Phi \text{ is phase flow of \cref{eqn:Hsys}}\right\},
	\end{align*}
	where $H$ is given by
	\begin{align*}
		H(p,q,t)= K(p,t) + V(q,t)
	\end{align*}
	for some $K,V \in C^{r+1}(U\times \R_0^+)$, if the activation function $\sigma$ is $r$-finite.
\end{theorem}
\begin{proof}
	Like in the proof of \Cref{the:la-sympnet}, we obtain the universal approximation property of NATLA-SympNets by showing that they can approximate the same systems as NATG-SympNets and obtain the desired result by using \Cref{the:naG-sympnet}. Without loss of generality, we focus on the lower gradient modules
	\begin{align*}
		u_{\text{low}}\left(h, \pq, t\right)&=
		\begin{pmatrix}
			p\\
			q + \frac{h}{m}K^T \diag(a) \sigma(K p + c t+ b)\\
			t + \frac{h}{m}
		\end{pmatrix}\\
		&=
		\begin{pmatrix}
			p\\
			q + \frac{h}{m}\sum_{i=1}^n K_i^T \diag(a_i) \sigma(K_i p + c_i t + b_i)\\
			t + \frac{h}{m}
		\end{pmatrix}
	\end{align*}
	and assume that the width of said modules is a multiple of $d$, i.e., $K=(K_1^T, ..., K_n^T)^T\in \R^{nd \times d},~a=(a_1^T, ..., a_n^T)^T, b=(b_1^T, ..., b_n^T)^T \in \R^{nd}$. If this requirement is not met, one could extend all of the parameters with zero rows until it is met without changing the output of the layer. Now let
	\begin{align*}
		s_i\left(h, \pq, t\right)=\begin{pmatrix}
			\begin{pmatrix}
				K_i^{-1} & 0\\
				0 & K_i^T
			\end{pmatrix}\begin{bmatrix}
				I & 0 \\
				\frac{h}{m} \tilde{\sigma}_{a_i, b_i, c_i} & I
			\end{bmatrix}
			\begin{pmatrix}
				K_i & 0\\
				0 & K_i^{-T}
			\end{pmatrix}\pq\\
			t + \frac{h}{mn}
		\end{pmatrix}
	\end{align*}
	for $i=1,...,n$. Obviously $s_i$ is a NATLA-SympNet and omitting the notation of $h$ as an argument of $s_i$, we observe
	\begin{align*}
		(s_n\circ ... \circ s_1)\left(h, \pq, t\right)=\begin{pmatrix}
			p \\
			q + \sum_{i=1}^n\frac{h}{m}K_i^T \diag(a_i)\sigma\left(K_i p + c_i \left(t + \frac{ih}{mn}\right) + b_i\right)\\
			t + \frac{h}{m}
		\end{pmatrix},
	\end{align*}
	which is almost the same as $u_{\text{low}}(h, (p^T, q^T)^T, t)$. A Taylor expansion of $s_{\text{low}}:=s_n\circ ... \circ s_1$ at $h=0$ yields
	\begin{align*}
		s_{\text{low}}\left(h, \pq, t\right)=&\begin{pmatrix}
			p\\ q\\ t
		\end{pmatrix}+ h 
		\begin{pmatrix}
			0 \\
			\sum_{i=1}^n \frac{1}{m}K_i^T \diag(a_i) \sigma(K_i p + c_i t  + b_i)\\
			\frac{1}{m}
		\end{pmatrix}\\
        &+ \frac{h^2}{2}
		\begin{pmatrix}
			0\\ \sum_{i=1}^n \partial_h^2 Q_i\left(\xi, \pq, t\right)\\0
		\end{pmatrix}\\
		=&u_{\text{low}}\left(h, \pq, t\right) + \frac{h^2}{2}
		\begin{pmatrix}
			0\\ \sum_{i=1}^n \partial_h^2 Q_i\left(\xi, \pq, t\right)\\0
		\end{pmatrix}
	\end{align*}
	with $\xi \in (0,h)$ and
	\begin{align*}
		Q_i\left(h,\pq, t\right):=\frac{h}{m}K_i^T \diag(a_i) \sigma\left(K_i p + c_i\left(t + \frac{ih}{mn}\right)+ b_i\right).
	\end{align*}
	Now one can follow the the same arguments as we used to derive \cref{eqn:higher_derivs} to show that
	\begin{align*}
		\norm{\sum_{i=1}^n \partial_h^2Q_i}_{C^r([0,T]\times \tilde{W}_t \times \tilde W;\R^{2d+1})}\leq \frac{C}{m^2}
	\end{align*} for some constant $C<\infty$ independent of $m$ and arbitrary compact $\tilde{W}_t, \tilde{W}$. So in the limit $m \to \infty$ NATLA-SympNets can approximate the layers of NATG-SympNets with $1/m^2$ accuracy. Since the same arguments can be made for the upper gradient modules, one obtains
	\begin{align}
		\label{eq:dif_g_la}
		\norm{\tilde{g}_{i,m}- g_{i,m}}_{C^r([0,T]\times \tilde{W}_t \times \tilde W;\R^{2d+1})}\leq \frac{\tilde{C}}{2m^2}
	\end{align}
	by using the same arguments as in the proof of \Cref{the:g-sympnet}, where $\tilde{C}< \infty$ is independent of $m$, $g_{i,m}$ is the composition of two gradient modules $u_\text{low}\circ u_\text{up}$ and $\tilde{g}_{i,m}=s_\text{low}\circ s_\text{up}$ is the composition of $s_\text{low}$ and it's counterpart $s_\text{up}$. Note that here $m$ is actually the number of layers, where as for the autonomous proofs $2m$ was the number of layers, since $g_{i,m}$ was the composition of an upper and a lower module.
	Let $\Phi_{1,i,m}$ be defined as in \Cref{the:compcor}. Now the triangle inequality yields
	\begin{align*}
		&\norm{\tilde{g}_{i,m}- \Phi_{1,i,m}}_{C^r([0,T]\times \tilde{W}_t \times \tilde W;\R^{2d+1})}\\
		&\leq \norm{\tilde{g}_{i,m}- g_{i,m}}_{C^r([0,T]\times \tilde{W}_t \times \tilde W;\R^{2d+1})}+ \norm{\Phi_{1,i,m}- g_{i,m}}_{C^r([0,T]\times \tilde{W}_t \times \tilde W;\R^{2d+1})},
	\end{align*}
	where both summands on the right can be bound by $\tilde{C}/(2m^2)$ for some $\tilde{C}<\infty$ because of \cref{eq:dif_g_la} and \Cref{the:naG-sympnet}. Hence,
	\Cref{the:compcor} implies that we can approximate the flow with $1/m$ accuracy.
\end{proof}
Contrary to the proof of \cite[Theorem 4]{JinZZetal20} we were not able to show that the layers of NATG-SympNets are equal to compositions of layers of NATLA-SympNets. We just showed that in the limit the number of layers $m\to \infty$ NATLA-SympNets can approximate NATG-SympNets with arbitrary accuracy.
\subsection{Limitations}
\label{subsec:lim}
The time-adaptive SympNets constructed in \Cref{subsec:td_SympNet} are able to approximate the flow of separable Hamiltonian systems. Unfortunately these SympNets are not able to approximate non-separable Hamiltonian systems.
To see this, we observe that for the true flow $\Phi$
\begin{align*}
	\partial_t \Phi(0,x)= \dot{x}(t_0)= J^{-1}\nabla H(x(t_0))
\end{align*}
holds. Now if wee look at a TSympNet $\psi=u_m \circ ... \circ u_1$, we observe
\begin{align*}
	\frac{\partial u_i(h,x)}{\partial h}= \begin{bmatrix}
		0 & \nabla V_i\\ 0 & 0
	\end{bmatrix}(x)=
	\begin{pmatrix}
		\nabla V_i(q) \\0
	\end{pmatrix} \qquad \text{ or } \qquad 
	\frac{\partial u_i(h,x)}{\partial h}= \begin{bmatrix}
		0 & 0\\ \nabla K_i & 0
	\end{bmatrix}(x)=\begin{pmatrix}
		0\\ \nabla K_i(p)
	\end{pmatrix}
\end{align*}
and
\begin{align*}
	\frac{\partial u_i(h,x)}{\partial x}=\begin{pmatrix}
		I & h \nabla^2 V_i(q)\\
		0 & I
	\end{pmatrix} \qquad \text{ or } \qquad 
	\frac{\partial u_i(h,x)}{\partial x}=\begin{pmatrix}
		I & 0\\
		h \nabla^2 K_i(p) & I
	\end{pmatrix}
\end{align*}
depending on whether $u_i, i=1,...,m$ is an upper or lower symplectic module. Therefore, we see that $\partial_x u_i(0,x)$ does not depend on $x$ and $u_i(0,x)=x$. The time derivative of $\psi$ at $h=0$ simplifies to
\begin{align}\label{eqn:crSN}
	\partial_h \psi(0,x)&=\frac{\partial u_m}{\partial x}(0,(u_{m-1}\circ... \circ u_1)(0,x))\frac{d}{dh}(u_{m-1}\circ...\circ u_1)(0,x)+ \frac{\partial u_m}{\partial h}(0,(u_{m-1}\circ... \circ u_1)(0,x))\notag \\
	&=\begin{pmatrix}
		I & 0\\
		0 & I
	\end{pmatrix}\frac{d}{dh}(u_{m-1}\circ...\circ u_1)(0,x)+\frac{\partial u_m}{\partial h}(0,x).
\end{align}
By repeating this argument we have
\begin{align*}
	\partial_h \psi(0,x)=\sum_{i=1}^{m}\frac{\partial u_i}{\partial h}(0,x)= \begin{pmatrix}
		-\nabla V(q)\\ \nabla K(p)
	\end{pmatrix}
\end{align*}
for $K:=\sum_{i}K_i$ and $V:=-\sum_{i}V_i$. Hence the gradient of the associated Hamiltonian $\tilde{H}$ is given by
\begin{align}
	\label{eq:inf_sep}
	\nabla \tilde{H}(x)= J \partial_h \psi(0,x)=J\ \begin{pmatrix}
		-\nabla V(q)\\ \nabla K(p)
	\end{pmatrix}=
	\begin{pmatrix}
		\nabla K(p)\\ \nabla V(q)
	\end{pmatrix}
\end{align}
and $\tilde{H}=K(p)+V(q)$ is separable.\\
Now we could think about other ways to make SympNets time-adaptive to extend their applicability to non-separable Hamiltonian systems. Let's generalize our TSympNet $\psi=u_m \circ ... \circ u_1$ by allowing time-dependent potentials, i.e.,
\begin{align*}
	u_i(h,x):= \begin{bmatrix}
		I & \nabla V_i(h,\cdot)\\
		0 & I
	\end{bmatrix}(x) \qquad \text{ or } \qquad
	u_i(h,x):= \begin{bmatrix}
		I & 0\\
		\nabla K_i(h,\cdot) & I
	\end{bmatrix}(x)
\end{align*}
for $i=1,..,m$. One of the major advantages of the TSympNets suggested in \cite{JinZZetal20} is $\psi(0,x)=x=\Phi(0,x)$. If we want to enforce this in our new structure, we need $\nabla V_i(0,q)=\nabla K_i(0,p)\equiv 0$. For this approach we observe
\begin{align*}
	\frac{\partial u_i(h,x)}{\partial h}= \begin{bmatrix}
		0 & \partial_h\nabla  V_i(h,\cdot)\\ 0 & 0
	\end{bmatrix}(x)=
	\begin{pmatrix}
		\partial_h \nabla V_i(h,q) \\0
	\end{pmatrix}
\end{align*}
or
\begin{align*}
	\frac{\partial u_i(h,x)}{\partial h}= \begin{bmatrix}
		0 & 0\\ \partial_h\nabla K_i(h,\cdot) & 0
	\end{bmatrix}(x)=\begin{pmatrix}
		0\\ \partial_h\nabla K_i(h,p)
	\end{pmatrix}
\end{align*}
and
\begin{align*}
	\frac{\partial u_i(h,x)}{\partial x}=\begin{pmatrix}
		I &  \nabla^2 V_i(h,q)\\
		0 & I
	\end{pmatrix} \qquad \text{ or } \qquad 
	\frac{\partial u_i(h,x)}{\partial x}=\begin{pmatrix}
		I & 0\\
		\nabla^2 K_i(h,p) & I
	\end{pmatrix}.
\end{align*}
With $\nabla V_i(0,q)\equiv 0$ we have
\begin{align*}
	\partial_{q_j}\nabla V_i(0,q)=\lim\limits_{h\to 0}\frac{\nabla V_i(0,q+ h e_j)- \nabla V_i(0,q)}{h}=0.
\end{align*}
Therefore $\nabla^2 V_i(0,q)\equiv 0$. The same argument holds for $\nabla^2 K_i$, and we have $\frac{\partial u_i(0,x)}{\partial x}=I_{2d}$. By repeating the argument from \cref{eqn:crSN}, again we have
\begin{align*}
	\partial_t \psi(0,x)=\sum_{i=1}^{m}\frac{\partial u_i}{\partial t}(0,x)= \begin{pmatrix}
		-\nabla V(q)\\ \nabla K(p)
	\end{pmatrix}
\end{align*}
with $V:=- \sum_i \partial_h V_i(0,\cdot)$ and $K:=\sum_i \partial_h K_i(0,\cdot)$. This leads to a separable Hamiltonian $\tilde{H}=K(p)+V(q)$ as well. The same arguments hold for TSympNets for non-autonomous Hamiltonian systems. So unfortunately there is no way to approximate non-separable Hamiltonian systems with TSympNets, if we want to keep the triangular structure of the symplectic modules, have a time step in every layer and enforce $u_i(0,x)=x$ strongly.\\
Note the arguments above do not hold for our new TLA-SympNets, while they are true for the original TLA-SympNets. We observe that this is not only a theoretical, but also a practical benefit in \Cref{subsec:lin}. The question if our new TLA-SympNets are universal approximators of general symplectic maps remains open for future work.\\
Obviously, the above arguments also hold for the non-autonomous case. Hence, the same limitations apply to NATG- and NATLA SympNets.
\subsection{Limitations of the original TLA-SympNets}
\label{subsec:lim_og_LA}
In this subsection we give our reasoning for choosing a different architecture for TLA-Sympnets as originally proposed in \cite{JinZZetal20}. Note that we did not derive a universal approximation theorem them. In the following we prove why this is impossible.
The linear layers of original time-dependent LA-SympNets are given by
\begin{align*}
	\mathcal{L}_n^{\text{up}}\left(h,
	\pq\right)&=
	\begin{pmatrix}
		I & 0/h S_n\\
		h S_n/0 & I
	\end{pmatrix}\hdots
	\begin{pmatrix}
		I & 0\\
		h S_2 & I
	\end{pmatrix}
	\begin{pmatrix}
		I & h S_1\\
		0 & I
	\end{pmatrix}
	\pq+b,\\
	\mathcal{L}_n^{\text{low}}\left(h,
	\pq\right)&=
	\begin{pmatrix}
		I & 0/h S_n\\
		h S_n/0 & I
	\end{pmatrix}\hdots
	\begin{pmatrix}
		I & h S_2\\
		0 & I
	\end{pmatrix}
	\begin{pmatrix}
		I & 0\\
		h S_1 & I
	\end{pmatrix}
	\pq+b.
\end{align*}
For simplicity we will focus on the $\mathcal{L}_n^{\text{up}}$ modules and assume $n$ to be even to simplify the notation
\begin{align*}
	\mathcal{L}_n\left(h,
	\pq\right)&=
	\begin{pmatrix}
		I & 0\\
		h S_n & I
	\end{pmatrix}\hdots
	\begin{pmatrix}
		I & 0\\
		h S_2 & I
	\end{pmatrix}
	\begin{pmatrix}
		I & h S_1\\
		0 & I
	\end{pmatrix}
	\pq+b.
\end{align*}
Now let $v_i$ be the composition of a linear layer with an activation layer
\begin{align*}
	v_i^{\text{up}}\left(h,\pq\right)&=\begin{bmatrix}
		I & h \sigma_{a_{q,i},b_{q,i}}\\
		0 & I
	\end{bmatrix}\mathcal{L}_n^{(i)}\left(h, \pq\right),\\
	v_i^{\text{low}}\left(h,\pq\right)&=\begin{bmatrix}
		I & 0 \\
		h \sigma_{a_{p,i},b_{p,i}} & I
	\end{bmatrix}\mathcal{L}_n^{(i)}\left(h, \pq\right).
\end{align*}
Then we observe
\begin{align*}
	\partial_h v_i^{\text{up}}\left(h, \pq\right)\bigg\vert_{h=0}&=
	\begin{bmatrix}
		0 & \sigma'_{a_{q,i},b_{q,i}}\\
		0 & 0
	\end{bmatrix}\pq+\sum_{j=1}^{n}\begin{pmatrix}
		0 & S_j^{(i)}/0\\
		0/S_j^{(i)} & 0
	\end{pmatrix}\pq=\begin{pmatrix}
		\tilde{S}_i q +\sigma'_{a_{q,i},b_{q,i}}(q)\\
		\hat{S}_ip
	\end{pmatrix},\\
	\partial_h v_i^{\text{low}}\left(h, \pq\right)\bigg\vert_{h=0}&=
	\begin{bmatrix}
		0 & 0\\
		\sigma'_{a_{p,i},b_{p,i}} & 0
	\end{bmatrix}\pq+\sum_{j=1}^{n}\begin{pmatrix}
		0 & S_j^{(i)}/0\\
		0/S_j^{(i)} & 0
	\end{pmatrix}\pq=\begin{pmatrix}
		\tilde{S}_iq\\
		\hat{S}_ip +\sigma'_{a_{p,i},b_{p,i}}(p)
	\end{pmatrix},
\end{align*}
with $\tilde{S}_i=\sum_{j=1}^{n/2}S_{2j-1}^{(i)},~\hat{S}_i=\sum_{j=1}^{n/2}S_{2j}^{(i)}$.  With the same derivations as for the other modules, we obtain
\begin{align*}
	\partial_x v_i^{\text{up}}\left(0, \pq\right)=\partial_x v_i^{\text{low}}\left(0, \pq\right) = I.
\end{align*}
If we compute the time derivative of the entire LA-SympNet, we arrive at
\begin{align*}
	\frac{d}{dh}\left(v_m^{low}\circ v_m^{\text{up}}\circ \hdots \circ v_1^{low}\circ v_1^{\text{up}}\right)\bigg\vert_{h=0}&=
	\sum_{i=1}^{m} \partial_h v_i^{\text{up}}\left(0,\pq\right)+\partial_h v_i^{\text{low}}\left(0,\pq\right)\\
	&=\sum_{i=1}^m\begin{pmatrix}
		\tilde{S}_i^{\text{up}}q + \tilde{S}_i^{low}q + \sigma'_{a_{q,i},b_{q,i}}(q)\\
		\hat{S}_i^{\text{up}}p + \hat{S}_i^{low}p + \sigma'_{a_{p,i},b_{p,i}}(p)
	\end{pmatrix}\\
	&=\begin{pmatrix}
		S_q q + \sum_{i=1}^m\sigma'_{a_{q,i},b_{q,i}}(q)\\
		S_p p + \sum_{i=1}^m\sigma'_{a_{p,i},b_{p,i}}(p)
	\end{pmatrix},
\end{align*}
with $S_q = \sum_{i=1}^m\tilde{S}_i^{low}+\tilde{S}_i^{\text{up}}$ and $S_p=\sum_{i=1}^m \hat{S}_i^{\text{up}}+\hat{S}_i^{low}$.
So the learned Hamiltonian always has the form
\begin{align*}
	\tilde{H}(p,q)=-\frac{1}{2} S_q q^2 + \frac{1}{2} S_p p^2 + \sum_{i=1}^m a_{p,i}^T\sigma(p+b_{p,i})- \sum_{i=1}^m a_{q,i}^T \sigma(q+b_{q,i}).
\end{align*}
Note that the functions in the sums do are not universal approximators of the potentials that we would need to have a general separable Hamiltonian, since the matrices inside the activation functions are missing \cite{HorSW90}. 
Since this severely limits the set of Hamiltonian systems whose flow that can be expressed by the orignal TLA-SympNets, we proposed a different architecture in \Cref{subsec:td_SympNet}.

%% file: experiments.tex
\section{Numerical Experiments}
\label{sec:exp}
In this section we validate our theoretical results with numerical experiments. In \Cref{subsec:pend} we show that our new LA-SympNet architecture can solve the pendulum problem, since this is the example of choice in \cite{JinZZetal20}. In \Cref{subsec:lin} we show that the theoretical limitation of the presented methods to separable Hamiltonian systems also holds in practice. We consider a very easy linear, but non-separable Hamiltonian system and observe that both TG-SympNets and the orignal TLA-SympNets fail at learning the system. In \Cref{subsec:ho} we show the necessity of our extensions for non-autonomous Hamiltonian systems. We observe that the autonomous methods fail, while the non-autonomous methods are able to make accurate predictions.\\
The training in all following subsections is performed by minimizing the mean squared error (MSE)
\begin{align*}
	L = \frac{1}{N}\sum_{i=1}^N\norm{\psi(x_i)-y_i}_2^2
\end{align*}
for $x_i,y_i$ in the training set $\mathcal{T}=\{(x_i,y_i)\}_{i=1}^N$ with the Adam optimizer \cite{KinB17}. We use $\tanh$ as our activation function and the architectures found in \Cref{tab:arch}. To perform the training we adapted the "Learner" framework (https://github.com/jpzxshi/learner).
\begin{table}
	\centering
	\begin{tabular}{|c|c|c|c|c|c|}
		\hline {Experiment} & NN Type & Layers & Sublayers & Width & Parameters \\  \hline \multirow{3}{*}{Pendulum} & TG-SympNet & 5 & N/A & 30 & 450 \\ \cline{2-6} & OG TLA-SympNet & 5 & 4 & N/A & 34 \\ \cline{2-6} & TLA-SympNet & 5 & 4 & N/A & 30 \\ \hline \multirow{3}{*}{Linear example} & TG-SympNet & 5 & N/A & 30 & 450 \\  \cline{2-6} & OG TLA-SympNet & 8 & 4 & N/A & 55 \\ \cline{2-6} & TLA-SympNet & 8 & 4 & N/A & 48 \\ \hline \multirow{4}{*}{Forced harmonic oscillator} & TG-SympNet & 6 & N/A & 20 & 360 \\ \cline{2-6} & TLA-SympNet & 240 & 2 & N/A & 960 \\ \cline{2-6}&  NATG-SympNet & 6 & N/A & 20 & 480 \\ \cline{2-6}&  NATLA-SympNet & 240 & 2 & N/A & 1200 \\\hline \end{tabular}
	\caption{Architecture of the SympNets used in the different experiments}
	\label{tab:arch}
\end{table}

\subsection{Pendulum}
\label{subsec:pend}
Like in \cite{JinZZetal20} we first consider the pendulum system with
\begin{align*}
	H(p,q)= \frac{1}{2}p^2 - \cos(q)
\end{align*}
to make sure that our new architecture can deal with the problem, which the original SympNets were able to handle. Also, we use the same kind of training and test datasets as in \cite{JinZZetal20}, i.e., the training dataset consists of $N=40$ tuples $\mathcal{T}=\{([x_i,h_i],y_i)\}_{i=1}^N$. The phase space coordinates $\{x_i\}_{i=1}^N$ are sampled from a uniform distribution over $[-\sqrt{2}, \sqrt{2}]\times [-\frac{1}{2}\pi, \frac{1}{2}\pi]$ (blue circles in \Cref{fig:pend}). The time steps $\{h_i\}_{i=1}^N$ are sampled from a uniform distribution over $[0.2,0.5]$. The labels $y_i$ are prediction of the flow $y_i=\Phi(h_i,x_i)$ (yellow points in \Cref{fig:pend}). Since we do not have an analytical representation of the flow $\Phi$, we use the 6th-order Störmer-Verlet Scheme \cite{HaiLW03} and make 10 steps of size $h_i/10$ to obtain $y_i$. The test dataset is a single trajectory with $k=100$ fixed time steps $h=0.1$ starting at $x_0=(1,0)^T$. All neural networks are trained for 50000 epochs with a learning rate of $10^{-3}$. We test TG-SympNets and both the original and new TLA-SympNets with the architectures found in \Cref{tab:arch}. Their prediction is generated by making $k$ forward evaluations of the respective SympNets.\\
In \Cref{fig:pend} we can see that all the SympNets are able to capture the dynamics of the pendulum system. This validates the results from \cite{JinZZetal20} and shows that our modified LA-SympNets can deal with the problem the original LA-SympNets were tested on even with less parameters.
\begin{figure}
	\centering
	\begin{subfigure}{.33\textwidth}
		\centering
		\includegraphics[width=.9\linewidth]{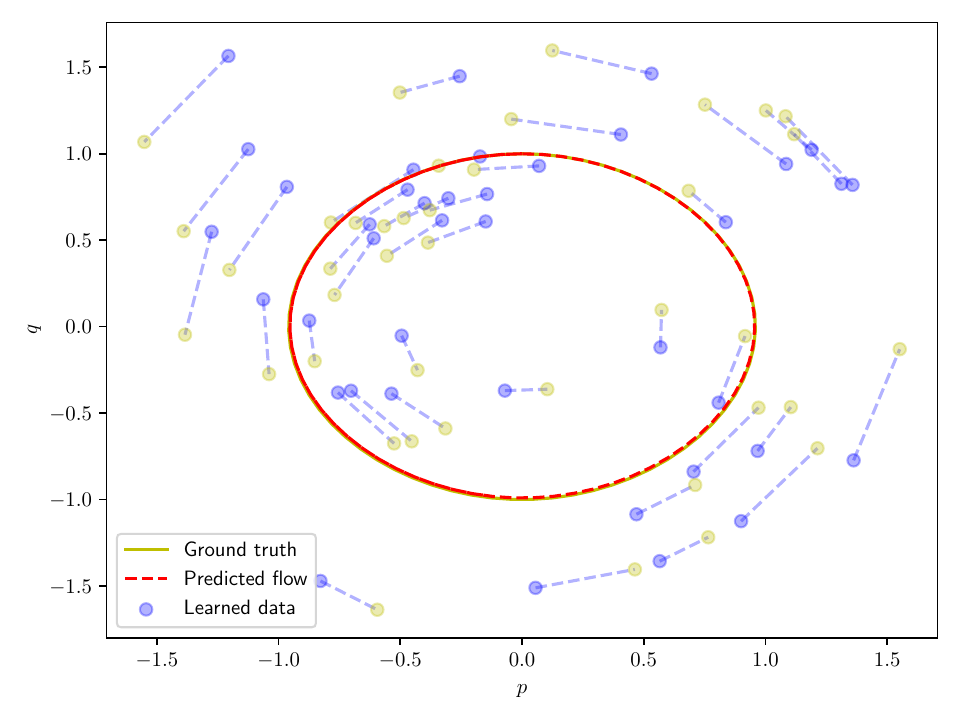}
		\caption{TG-SympNet}
		\label{fig:pend_G}
	\end{subfigure}%
	\begin{subfigure}{.33\textwidth}
		\centering
		\includegraphics[width=.9\linewidth]{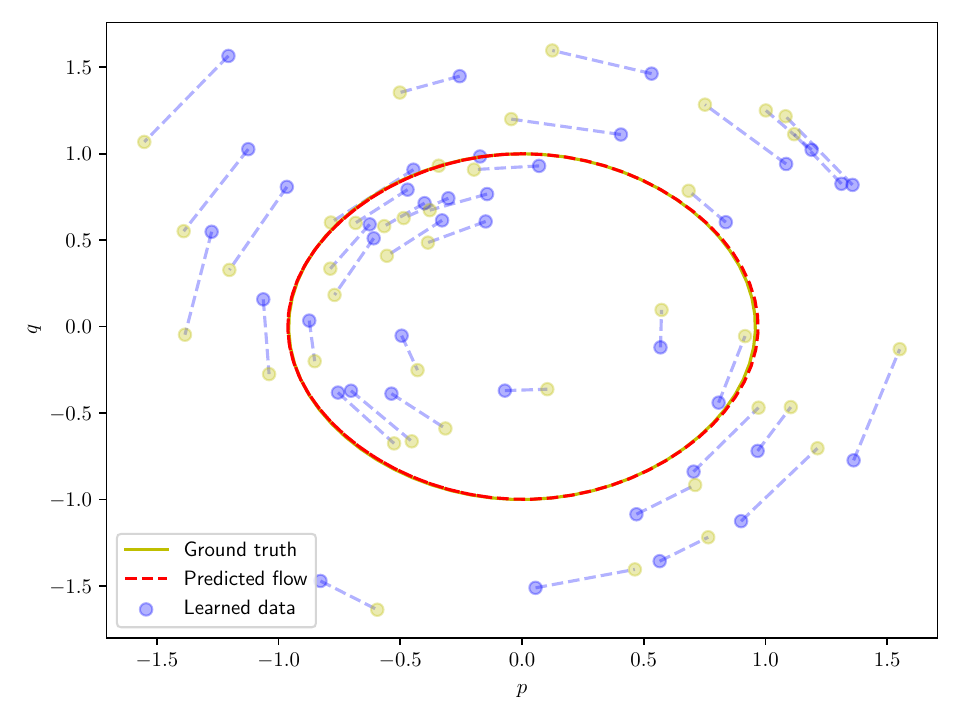}
		\caption{Original LA-SympNet}
		\label{fig:pend_LA}
	\end{subfigure}
	\begin{subfigure}{.33\textwidth}
		\centering\includegraphics[width=0.9\linewidth]{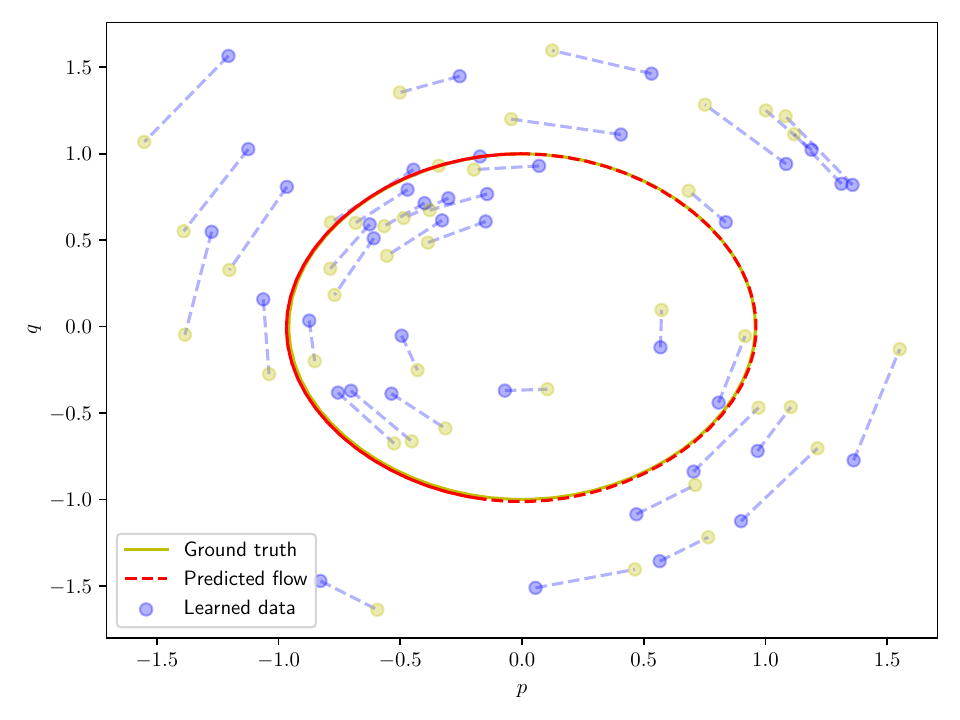}
		\caption{TLA-SympNet}
	\end{subfigure}
	\caption{Pendulum example $H(p,q)=\frac{1}{2}p^2- \cos(q)$. Blue and yellow points with dashed lines represent the training data. The solid yellow line is the test trajectory.}
	\label{fig:pend}
\end{figure}

\subsection{Linear example}
\label{subsec:lin}
Now we consider a simple non-separable linear example with the Hamiltonian given by
\begin{align*}
	H(p, q)=\frac{1}{2}p^2 + 0.4 pq + \frac{1}{2}q^2.
\end{align*}
The point of this experiment is to show that the limitations found in \Cref{subsec:lim} are not only of theoretical nature, but also have practical implications. Recall that both the original TLA- and TG-SympNets should not be able to deal well with non-separable Hamiltonian systems, because they always infer separable Hamiltonians according to \cref{eq:inf_sep}.\\
Training and test datasets are generated the same way as in the pendulum example in \Cref{subsec:pend}. The used architectures can be found in \Cref{tab:arch} again.
\begin{figure}
	\centering
	\begin{subfigure}{.33\textwidth}
		\centering
		\includegraphics[width=0.9\linewidth]{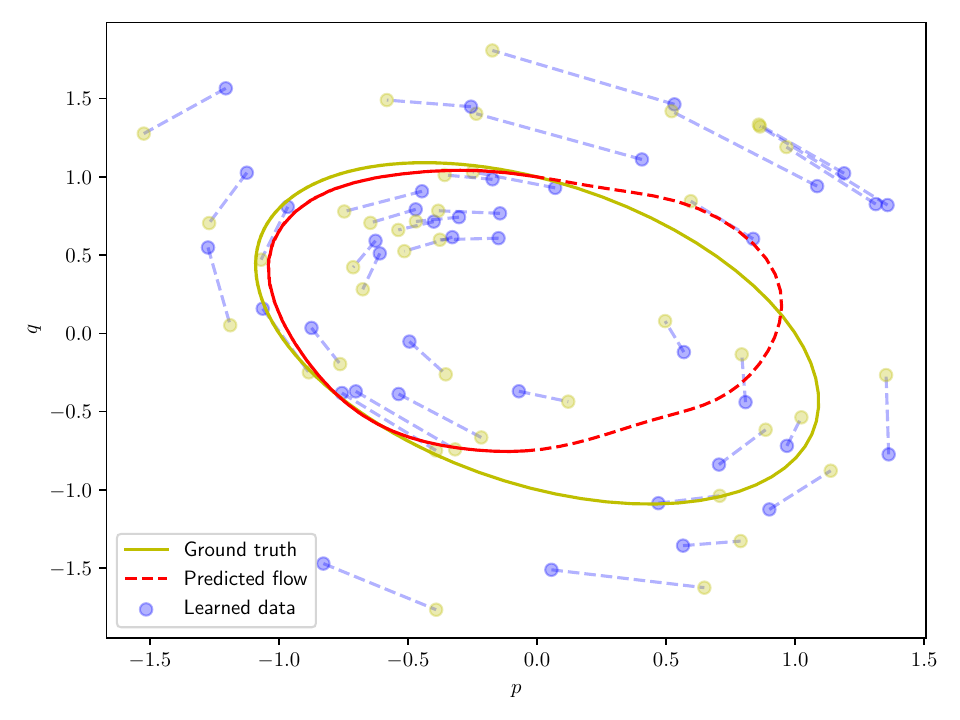}
		\caption{TG-SympNet}
		\label{fig:lin_G}
	\end{subfigure}%
	\begin{subfigure}{.33\textwidth}
		\centering
		\includegraphics[width=0.9\linewidth]{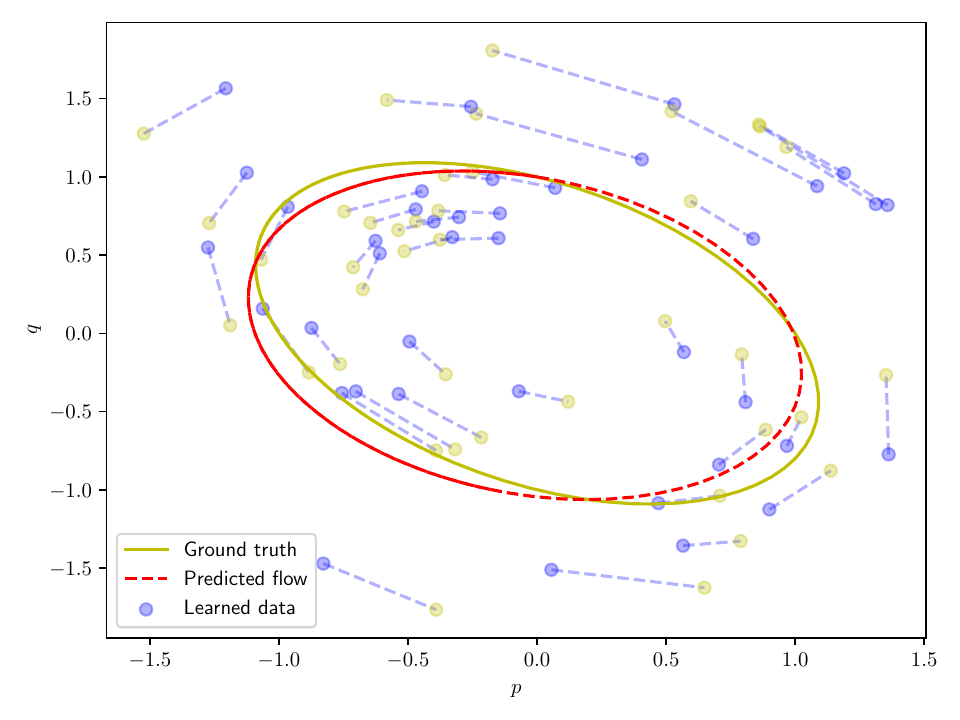}
		\caption{Original TLA-SympNet}
		\label{fig:lin_og_LA}
	\end{subfigure}%
	\begin{subfigure}{.33\textwidth}
		\centering\includegraphics[width=0.9\linewidth]{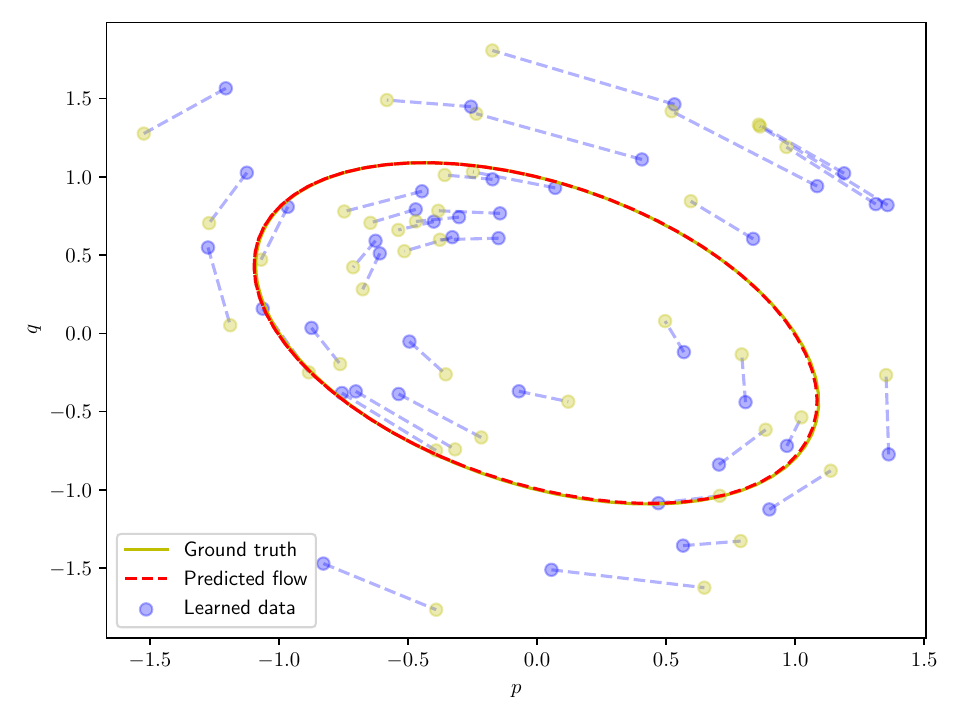}
		\caption{TLA-SympNet}
		\label{fig:lin_LA}
	\end{subfigure}
	\caption{Linear example $H(p,q)=\frac{1}{2}p^2 + 0.4 pq + \frac{1}{2} q^2$. Blue and yellow points with dashed lines represent the training data. The solid yellow line is the test trajectory.}
	\label{fig:lin}
\end{figure}
In \Cref{fig:lin} we observe that our theoretical results from \Cref{subsec:lim} hold true in practice. While both the original TG-SympNet and the original TLA-SympNet fail to capture the dynamics correctly, our new TLA-SympNets are able to do so even though the considered system is not separable. This results suggests that our new TLA-SympNet architecture might be able to approximate the flow of non-separable Hamiltonian systems as well. Investigating them further for non-separable Hamiltonian systems or developing a new architecture by adapting them could be a promising.

\subsection{Forced harmonic oscillator}
\label{subsec:ho}
The last example we consider is the harmonic oscillator with an external sinusiodal force. The non-autonomous Hamiltonian is given by
\begin{align*}
	H(p, q, t)=\frac{1}{2}p^2+\frac{1}{2}\omega_0^2 q^2 - F_0 \sin(\omega t)q.
\end{align*}
Fortunately, the corresponding Hamiltonian system can be solved analytically. The analytical solution is given by
\begin{align*}
	p(t)&=\left(p_0 -\frac{\omega F_0}{\omega_0^2 - \omega^2}\right)\cos(\omega_0 t) -q_0 \omega_0 \sin(\omega_0 t)+ \frac{\omega F_0}{\omega_0^2 - \omega^2}\cos(\omega t),\\
	q(t)&=q_0 \cos(\omega_0 t) + \left(\frac{p_0}{\omega_0} -\frac{\omega F_0}{\omega_0(\omega_0^2 - \omega^2)}\right)\sin(\omega_0 t) + \frac{F_0}{\omega_0^2 - \omega^2}\sin(\omega t),
\end{align*}
with initial conditions $p(0)=p_0,~q(0)=q_0$. Hence, we do not need to use a numerical solver to generate our training and test data. The training data set is given by $N=1600$ tuples $\mathcal{T}=\{[x_i,t_i,h_i],[y_i,\tilde{t}_i]\}_{i=1}^N$, where $\{x_i\}_{i=1}^N$ is sampled from a uniform distribution over $[-3.5,2]\times[-4,4]$. The times $\{t_i\}_{i=1}^N$ and time-steps $\{h_i\}_{i=1}^N$ are also sampled from uniform distributions over $[0,16]$ and $[0,0.3]$ respectively. The labels $\{y_i\}_{i=1}^N$ are given by 
$y_i=(p(t_i+h_i)^T, q(t_i+h_i)^T)^T$. The test data is given by the trajectory of the analytical solution with initial value $(p_0,q_0)^T=(-0.2, -0.5)^T$ evaluated equidistant times $t_i=ih$ with $h=0.2$ for $i=1,...,k,~k=80$. All neural networks are trained for 150000 epochs with a learning rate of $10^{-3}$.\\
\Cref{fig:ho_G,fig:ho_LA} show that both TG-SympNets as well as TLA-SympNets are not able to capture the dynamics of the harmonic oscillator with an external force. This was expected, because the test trajectory intersects with itself. So the SympNets would need to give different predictions at the same phase space coordinate $(p,q)$ depending on the time $t$. Since time is not an input of these SympNets, this is impossible.\\
\Cref{fig:ho_TG} shows that NATG-SympNets are able to capture the dynamics well. The mentioned intersection is not a problem for NATG-SympNets, since it lives on the $(p,q,t)$ phase space and the intersection only happens in the $(p,q)$ space.\\
Note that the proof of \Cref{the:naLA-sympnet} suggests that NATLA-SympNets need more (width of the gradient modules as many) layers than G-SympNets as does the one of \cite[Theorem 4]{JinZZetal20}. In the previous autonomous experiments as well as in the ones of \cite{JinZZetal20} this was not observed. But here we see that the number of layers needed for the NATLA-SympNet is much larger then for NATG-SympNets, see \Cref{tab:arch}. This results in more expensive training and evaluation, making NATG-SympNets the far superior method for non-autonomous Hamiltonian systems. Nevertheless, \Cref{fig:ho_TLA} confirms the theoretical approximation capabilities of NATLA-SympNets established in \Cref{the:naLA-sympnet}.
\begin{figure}[ht]
	\centering
	\begin{subfigure}{.5\textwidth}
		\centering
		\includegraphics[width=.9\linewidth]{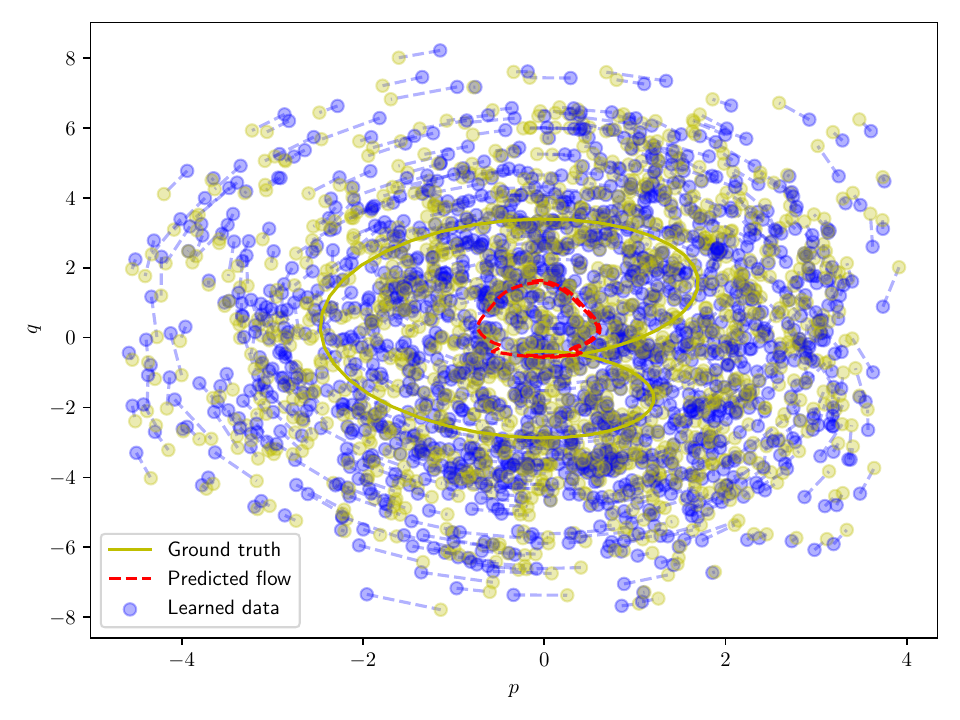}
		\caption{TG-SympNet}
		\label{fig:ho_G}
	\end{subfigure}%
	\begin{subfigure}{.5\textwidth}
		\centering
		\includegraphics[width=.9\linewidth]{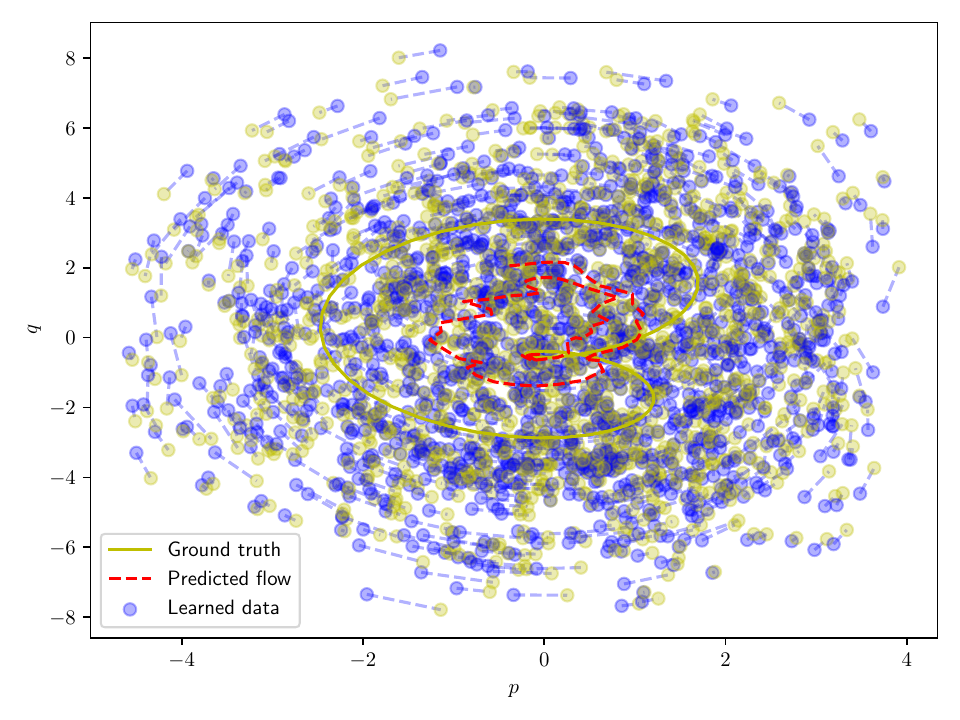}
		\caption{TLA-SympNet}
		\label{fig:ho_LA}
	\end{subfigure}
	\begin{subfigure}{.5\textwidth}
		\centering
		\includegraphics[width=0.9\linewidth]{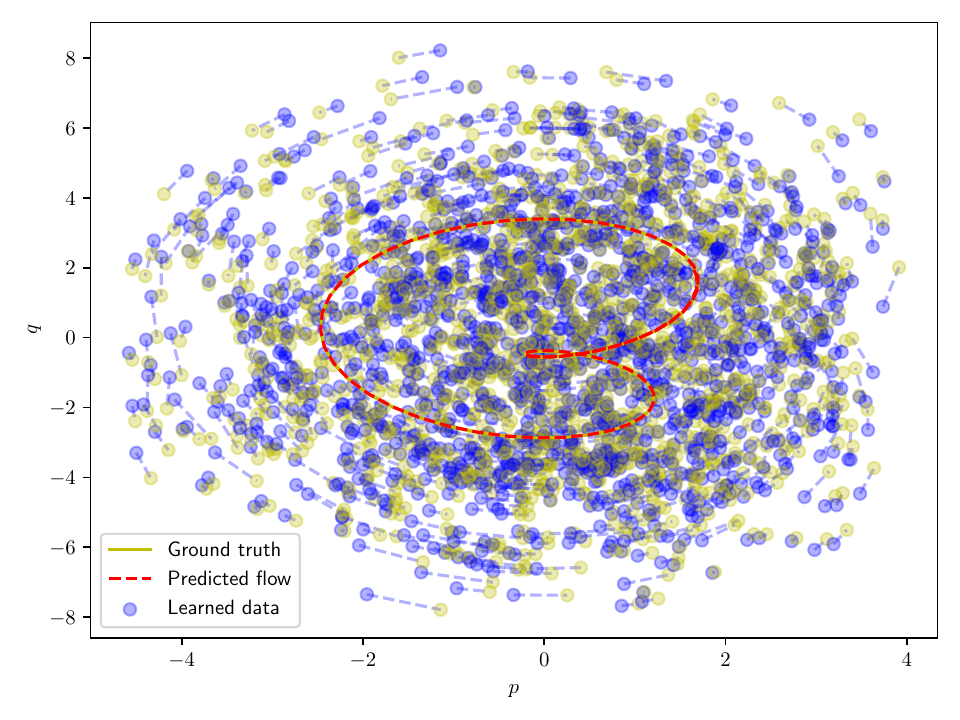}
		\caption{NATG-SympNet}
		\label{fig:ho_TG}
	\end{subfigure}%
	\begin{subfigure}{.5\textwidth}
		\centering
		\includegraphics[width=0.9\linewidth]{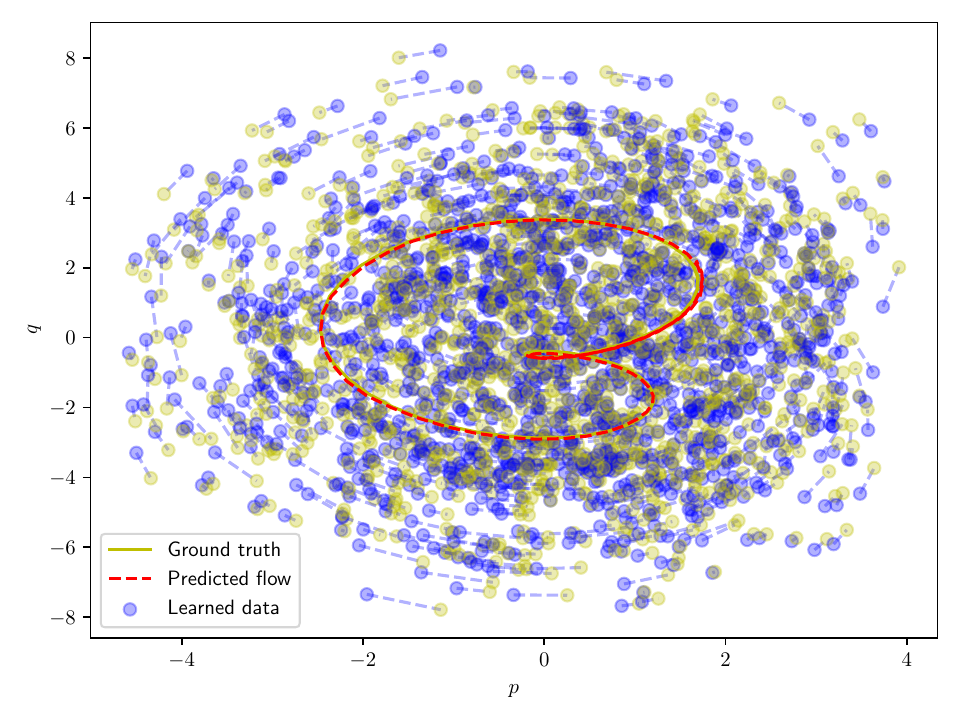}
		\caption{NATLA-SympNet}
		\label{fig:ho_TLA}
	\end{subfigure}
	\caption{Forced harmonic oscillator $H(p,q,t)=\frac{1}{2}p^2+\frac{1}{2}\omega_0^2 q^2 - F_0 \sin(\omega t)q$. Blue and yellow points with dashed lines represent the training data. The solid yellow line is the test trajectory.}
	\label{fig:ho}
\end{figure} 
\section*{Code and Data Availability Statement}
\label{sec:code}
The relevant code and data for this work has been archived within the Zenodo repository \cite{JanB25software}.

%% file: conclusion.tex
\section{Conclusion and Outlook}
\label{sec:conc}
In this work we first corrected the proof of the symplectic polynomial approximation theorem \cite[Theorem 2]{Tur03}, which is an important theorem for the approximation qualities of SympNets and HénonNets and generalized HNNs. We developed a new time-adaptive LA-SympNet architecture, we coined TLA-SympNets and established a universal approximation theorem for them as well as TG-Sympnets limited to separable Hamiltonian systems. This limitation is not an artifact of our proofs, but a direct result of the used architecture shown in \Cref{subsec:lim} theoretically and \Cref{subsec:lin} experimentally. Additionally, we proposed a novel intrinsically symplectic, time-adaptive neural network architecture to handle non-autonomous Hamiltonian systems in NATG-SympNets and NATLA-SympNets. We were able to show that they are universal approximators for the flows of separable non-autonomous Hamitlonian systems. Lastly, we verified our theoretical findings on three numerical experiments.\\
The obvious open research question is how to adapt T-SympNets such that they will be able to handle non-separable Hamiltonian systems as well. Furthermore, we are interested in creating time-adaptive extensions for HénonNets or even generalizing our findings to generalized HNNs.